\documentclass{article}

     \PassOptionsToPackage{round, numbers, compress}{natbib}
\usepackage[preprint]{neurips2026/neurips_2026}

\usepackage[utf8]{inputenc} 
\usepackage[T1]{fontenc}    
\usepackage{hyperref}       
\usepackage{url}            
\usepackage{booktabs} 
\usepackage{amsmath, amssymb, amsthm, amsfonts} 
\usepackage{nicefrac}       
\usepackage{microtype}      
\usepackage{xcolor}         
\usepackage{mathtools} 
\usepackage{tikz} 
\usepackage{algorithm}
\usepackage{algpseudocode}
\usepackage{graphicx}
\usepackage{subcaption}
\usepackage{hyperref}
\usepackage{makecell}
\usepackage{nicefrac}       
\usepackage{microtype}      
\usepackage{soul}
\usepackage{cleveref} 
\usepackage{enumitem}
\usepackage{dsfont,bm}
\hypersetup{colorlinks=true, linkcolor = blue,
     anchorcolor = blue,
     citecolor = blue,
     filecolor = blue,
     urlcolor = blue}

\def\redb{\color{red}\bf}
\def\blueb{\color{blue}\bf}
\def\algbound{{{\blueb Bound regression}}}
\def\algsamp{{{\redb Surrogate sampling}}}

\DeclareMathOperator*{\argmin}{arg\,min}
\DeclareMathOperator*{\argmax}{arg\,max}

\newtheorem{thmprop}{Proposition}
\newtheorem{thmthm}{Theorem}

\newtheorem{thmasmp}{Assumption}
\newtheorem*{thmidasmp*}{Identifying assumptions}

\theoremstyle{definition}

\newtheorem*{thmrem*}{Remark}
\newtheorem*{thmprop*}{Proposition}

\makeatletter
\newtheorem*{rep@theorem}{\rep@title}
\newcommand{\newreptheorem}[2]{%
\newenvironment{rep#1}[1]{%
 \def\rep@title{#2 \ref{##1} (Restated)}%
 \begin{rep@theorem}}%
 {\end{rep@theorem}}}
\makeatother
\theoremstyle{plain}
\newreptheorem{theorem}{Theorem}
\newreptheorem{prop}{Proposition}
\newreptheorem{lemma}{Lemma}
\newreptheorem{cor}{Corollary}
\theoremstyle{definition}
\newreptheorem{def}{Definition}

\newcommand\independent{\protect\mathpalette{\protect\independenT}{\perp}}
\def\independenT#1#2{\mathrel{\rlap{$#1#2$}\mkern2mu{#1#2}}}
\def\indep{\independent{}}

\def\med{\operatorname{med}}
\def\leaf{\operatorname{leaf}}
\def\proxy{\operatorname{proxy}}

\def\E{\mathbb{E}}

\def\hh{\hat{h}}
\def\mc#1{\mathcal{#1}}
\def\e{\epsilon}
\DeclarePairedDelimiter\set{\{}{\}}

\def\bbR{\mathbb{R}}

\def\cF{\mathcal{F}}

\title{Learning plug-in surrogate endpoints for randomized experiments}

\author{%
  Alessandro-Umberto Margueritte \thanks{Equal contribution.} \\
  AstraZeneca\\
  Chalmers University of Technology  \\ and University of Gothenburg \\
  \And
  Ahmet Zahid Balc{\i}o\u{g}lu \footnotemark[1]\\
  Chalmers University of Technology  \\ and University of Gothenburg \\
\texttt{ahmet.balcioglu@chalmers.se} \\
  \AND
  Jesse H. Krijthe \\
  TU Delft \\
  \And
  Dave Zachariah \\
  Uppsala University \\
  \And
  Fredrik D. Johansson \\
  Chalmers University of Technology  \\ and University of Gothenburg \\
  }
\begin{document}
\maketitle
\begin{abstract}
Surrogate endpoints are used in place of long-term outcomes in randomized experiments when observing the real outcome for a large enough cohort is prohibitively expensive or impractical. A short-term surrogate is good if the result of an experiment using the surrogate is predictive of the result of a hypothetical study using the real outcome. Much attention has been paid to formalizing this property in causal terms, but most criteria are unidentifiable and cannot be turned into practical algorithms for learning surrogate endpoints from data. To address this, we study plug-in composite surrogates, functions of post-treatment variables that may be substituted directly for the primary outcome in a randomized experiment. We propose two methods for learning plug-in surrogates that maximize effect predictiveness, and characterize the possibility of finding endpoints that yield unbiased effect estimates in representative scenarios. Finally, in both synthetic experiments with known effects and in data from a real-world experiment, we find that our method, based on directly modeling the surrogate effect, returns plug-in endpoints more predictive of the primary effect than established methods.
\end{abstract}
\section{Introduction}
\label{sec:introduction}
In many experiments, such as a clinical trial or A/B tests, the long-term effects of interventions are difficult to estimate since the outcomes of interest rarely occur within the study period. For instance, nutrition science is notorious for its many examples~\citep{yetley2017surrogate}, where the risk of cardiovascular disease is an outcome of interest, but few study participants can maintain a prescribed diet long enough to see disease develop. Instead, experimentalists have turned to measurable \emph{surrogates} of the primary outcomes---variables hoped to be informative of the efficacy of the intervention \emph{and} yield reliable effect estimates on the primary outcome with tractable trial lengths and cohort sizes.

Surrogate endpoints have been identified, variously, by known mechanistic relationships~\citep{fleming1996surrogate,temple1999surrogate} or statistical associations~\citep{prentice1989surrogate,Baker2018-FiveCriteriaUsingSurrogate} with the primary outcome, as mediators of the causal effect~\citep{chen2007criteria,robins1992identifiability}, or as variables on which the causal effect of the intervention is predictive of the causal effect on the primary outcome~\citep{prentice1989surrogate,chen2007criteria,meyvisch2020surrogate,burzykowski2005evaluation}. Formalizations of these conditions~\citep{chen2007criteria,joffe2009related,vanderweele2013surrogate} have opened the possibility for the data-driven search for new surrogates based on observational~\citep{athey2025surrogate} or past experimental data~\citep{tripuraneni2024choosing,wang2022surrogate,zhang2023evaluating}, to design experiments. 

Learning surrogate endpoints from observational data requires the causal identification~\citep{pearl2009causality} of surrogacy criteria as learning objectives, as well as ensuring the transportability~\citep{pearl2011transportability} of estimates to future experiments. However, most criteria are based on matching the \emph{average} causal effects of surrogate and outcome~\citep{vanderweele2013surrogate}, quantities that are sensitive to distribution shift, and are insufficient for personalized decision-making. Under appropriate assumptions, this can be overcome by learning a set of nuisance functions that capture both effect heterogeneity, confounding due to pre-treatment variables, and transportability, and using these for statistical estimates in the trial~\citep{athey2025surrogate}. However, using such functions as trial endpoints is undesirable, as they depend directly on pre-treatment variables.  

In this work, we study learning \emph{plug-in surrogates} from observational data---functions of only post-treatment variables that can be used in experiments or clinical trials as direct substitutes for the primary outcome. We contribute
\vspace{-0.75em}
\begin{itemize}
    \item Desiderata and a corresponding learning objective for plug-in surrogates, aimed at capturing as much heterogeneity in the causal effect on the outcome as possible. 
    \item An analysis of scenarios where good plug-in surrogates exist, and lead to effect estimates that are unbiased for the primary outcome, and where they don't.
    \item Two algorithms designed to minimize our proposed objective and an empirical evaluation on synthetic and real-world data, highlighting their strengths and weaknesses compared to established methods.
\end{itemize}
We find that our algorithm, which directly models the surrogate effect by sampling counterfactual surrogate values from estimated densities, consistently returns plug-in endpoints that are more predictive of the primary-outcome effect than those of baseline methods.
\section{Problem setup \& related work}
\label{sec:background}
We study surrogates for the primary outcome $Y  \in \mathbb{R}$ in a randomized experiment, aimed at estimating the causal effect of a binary intervention (treatment) $T$ on $Y$ in the context of a set of pre-treatment variables $X \in \bbR^k$. We focus on experiments where $Y$ cannot be observed directly due to practical limitations, such as the outcome being too long-term to observe (e.g., cardiovascular disease or death), and we must rely on a set of post-treatment variables $S \in \bbR^d$ (e.g., metabolic markers) as candidate surrogates for $Y$.  

Our goal is to \emph{learn} a composite endpoint $f(S)$, a function $f : \bbR^d \rightarrow \bbR$ of the candidate surrogates $S$, that can serve \emph{as a plug-in replacement for} $Y$ in a typical randomized experiment. First, $f$ is learned from retrospective data from an observational population, denoted $P=o$. Second, an experiment $(P=e)$ is performed, in with each instance $i$ is randomly assigned a treatment $t_i \in \{0,1\}$. The post-treatment surrogate vector $s_i \in \bbR^d$ is observed and used to compute $f_i = f(s_i)$, which is in turn used to estimate the treatment effect on $f$, as if $f_i$ were the primary outcome $y_i$. The trial proceeds exactly as it would have if $Y$ were observed, and conclusions about the effects of $T$ on $Y$ are drawn based on the effects of $T$ on $f(S)$. For this to be a sound and effective strategy, $f(S)$ must satisfy the following desiderata:
\begin{enumerate}[label={D\arabic*}]
    \item The intervention $T$ has a causal effect on the composite endpoint $f(S)$ that is measurable in an experiment \label{des:effect}
    \item The causal effect of $T$ on $f(S)$ is predictive of the causal effect of $T$ on the primary endpoint $Y$, ideally capturing important effect heterogeneity \label{des:pred}
    \item The composite $f(S)$ is an interpretable function of post-treatment variables that is learned \emph{before} the trial\label{des:int}
\end{enumerate}%

We expand on these desiderata below, in relation to previous works. For an excellent overview, see \citet{meyvisch2020surrogate}.

 First, we let $S(0), S(1), Y(0), Y(1)$ denote the full vector of \emph{potential outcomes}~\citep{rubin1974estimating} of  $S$ and $Y$, corresponding to interventions $t \in \{0,1\}$ on $T$.
The causal effects of $T$ on the surrogate set and primary outcome are defined by
$$
\Delta_S = S(1) - S(0) \;\;\mbox{ and }\;\; \Delta_Y = Y(1) - Y(0)~.
$$
We refer to these as the \emph{surrogate set effect} and \emph{primary effect}, respectively. For a composite endpoint function $f : \bbR^d \rightarrow \bbR$, we define $\Delta_{f(S)} = f(S(1)) - f(S(0))$ and refer to this as the \emph{surrogate effect}.

\paragraph{Regression of outcomes on post-treatment variables may return confounded plug-in surrogates.}\mbox{}\\
A common strategy for finding plug-in surrogates among multiple candidate variables is to fit a regression of $Y$ onto $S$, producing an estimate $f(S) \approx \E[Y \mid S]$ and using this as a composite score or to select a single variable as surrogate~\citep{coley2020dementia,esposito2004effect,martinez2015benefits,keys1984seven,chan2006apolipoproteins}. 
Often, a linear model $f(S) = \theta^\top S$ is used either as the score itself, or to select a variable $f(S) = S_j = \argmax_{j}|\theta_j|$, based on the sizes of the fitted coefficients.

The strengths and weaknesses of the regression strategy are well known~\citep{fleming_demets_1996surrogate,ciani2023framework,baker2003perfect,Baker2018-FiveCriteriaUsingSurrogate}.
Without confounding and with perfect mediation, i.e., when the causal graph contains only $T \rightarrow S \rightarrow Y$, it yields a surrogate whose effect $\tau_{f(S)}$ is predictive of the effect on the primary outcome. However, this precludes both direct effects of $T$ on $Y$, and common causes of $S$ and $Y$. If $S$ and $Y$ are both causally affected by variable $X$, as in Figure~\ref{fig:cdags4}, the outcome regression strategy will lead to confounded effect estimates in general. We show how this bias emerges and give an example scenario in Appendix~\ref{app:confounded_regression}. %
Moreover, when $S$ contains variables that are predictive of $Y$ but unaffected by $T$, variable selection or constrained/regularized models can fail, as they may favor associations that are irrelevant for the causal effect.
To avoid this, we aim to learn composite surrogates $f$ that \emph{capture the causal effect of $T$ on $Y$} (\ref{des:effect}).

\paragraph{Learning to match average causal effects leads to underspecified surrogate endpoints.}\mbox{}\\
The causal effect on a composite surrogate $f(S)$ is predictive (\ref{des:pred}) of the primary effect if $\Delta_Y$ is informative of $\Delta_{f(S)}$. A surrogate with this property enables consistent decision-making to improve $Y$, based on estimates of $\Delta_{f(S)}$ or its aggregates. If \ref{des:pred} is false, a trial may find no effect on the surrogate, even though there is an effect on the primary outcome. The reverse error can be equally costly, for example, if the effect on a surrogate does not replicate in a costly long-term trial where the main outcome is measured.

Regrettably, the joint distribution $p(\Delta_Y, \Delta_{f(S)})$ is unidentifiable (see Appendix~\ref{app:identifiability}), as is e.g., the mutual information $I(\Delta_{f(S)}; \Delta_Y)$~\citep{meyvisch2020surrogate} and correlation $\rho(\Delta_{f(S)}, \Delta_Y)$, prohibiting their use in learning $f$. Even with infinitely many randomized trials, we could not learn or validate a surrogate $f(S)$ based on such criteria. In fact, most surrogacy criteria in the literature are unidentifiable~\citep{vanderweele2013surrogate}.
Fortunately, most experiments have a more modest goal: to estimate the average (primary) treatment effect (ATE) in the experiment population, 
$$
\tau^e_Y = \E[\Delta_Y \mid P=e]~.
$$
The trial average surrogate effect $\tau^e_{f(S)} = \E[\Delta_{f(S)} \mid P=e]$ and observational-study counterparts $\tau^o_Y$ and $\tau^o_{f(S)}$ are defined analogously, the latter two conditioned on $P=o$. Henceforth, we let $p^o, p^e$ and $\E^o, \E^e$ denote densities and expectations under the observational and experimental populations, respectively, leaving out conditioning on $P$. 

Most surrogacy criteria are based on the possibility to recover the primary-outcome ATE using the surrogate~\citep{prentice1989surrogate,freedman1992statistical,frangakis2002principal}. For example, \citet{chen2007criteria} defined a \emph{consistent} surrogate $S$ (in the binary case) as a mediator for which the sign of the product $\tau^e_S \cdot \tau^e_{S \rightarrow Y}$ is equal to the sign of $\tau^e_Y$, where $\tau^e_{S \rightarrow Y}$ is the average causal effect of $S$ on $Y$. This rules out the ``surrogate paradox'', where $T$ and $S$ are positively associated, $S$ and $Y$ are positively associated, but $T$ has a negative causal effect on $Y$~\citep{chen2007criteria}. For recent works on evaluating surrogates, see Appendix~\ref{app:related}. 

ATE-based criteria suggest searching for a composite $f(S)$ to match the trial ATEs of the surrogate and primary outcome: $\tau^e_{f(S)} \approx \tau^e_Y$. Finding such a function  would allow its use as a plug-in replacement for $Y$ and support \emph{uniform} policy decisions based on experiment outcomes. However, this strategy has important caveats. 
First, it makes for a weak learning signal, leaving $f$ woefully \textbf{underspecified}. For example, under mild assumptions, there is a linear surrogate function $f(S) = \alpha S_j$, of any single variable $S_j$ affected by $T$, such that $\tau^e_{f(S)} = \tau^e_Y$ (see Appendix~\ref{app:ATE_matching}). Second, it \textbf{does not support personalized decision-making}, as it does not capture effect heterogeneity. Finally, the ATE is \textbf{not transportable} in general, $\tau^e_Y \neq \tau^o_Y$, as the distribution of contexts $X$ (e.g., population of subjects) vary, $p^e(X) \neq p^o(X)$, making the trial ATE difficult to learn in observational data. 

\paragraph{The surrogate index addresses transportability and personalization but depends on pre-treatment variables.}\mbox{}\\
To support personalized or contextual decision-making, it is suitable to estimate the \emph{conditional average treatment effect} (CATE) with respect to strata $X=x$, 
$$\tau^e_Y(x) = \E^e[\Delta_Y \mid X=x]~,$$
and act according the preferred treatment for each stratum.
We define $\tau^e_{f(S)}(x)$, $\tau^o_Y(x)$, and $\tau^o_{f(S)}(x)$ analogously. To simplify notation, $X$ will take the roles both of potential confounding variables and treatment effect moderators. In general, these could be separated into two distinct variables.

The \emph{surrogate index}, proposed by \citet{athey2025surrogate}, removes several of the limitations of ATE-matching strategies by estimating and transporting the nuisance function 
\begin{equation}\label{eq:surrogate_index}
h^o(x, s) \coloneqq \E^o[Y \mid X=x, S=s]~.
\end{equation}
Under ignorability, Prentice's surrogacy condition, and comparability, all to be defined formally later, it holds that
$$
\tau_Y^e(x) = \E^e[h^o(x, S) \mid x, T=1] - \E^e[h^o(x, S) \mid x, T=0]
$$
and, by definition, $\tau_Y^e = \E^e[\tau_Y^e(X)]$. Thus, both the trial ATE and CATE can be identified by learning $h^o(x, s)$ in observational data and averaging its predictions in the trial.

However, several properties of the surrogate index $h^o(x, s)$ make it \textbf{unsuitable as a plug-in surrogate} in clinical trials (\ref{des:int}). First, $h^o(x,s)$ depends directly on pre-treatment variables $X$ that cannot mediate the effect of $T$ on $Y$. This is generally not accepted for clinical trial endpoints~\citep{friedman2015fundamentals}, as pre-treatment variables are also used for trial inclusion criteria, and could be used to manipulate the estimated effect. Second, using the surrogate index relies on statistical quantities only estimable in trial data, such as the trial propensity $p(P=e \mid X, S)$, which prevents using $h^o(x,s)$ in place of $Y$ in a difference-in-means estimate of $\tau_Y^e$. Third, a surrogate must be such that a governing authority will accept the evidence provided by a trial using it as endpoint. Most likely, $f$ needs to belong to an interpretable model family, but the surrogate index analysis does not consider approximate estimates $h^o(x,s) \approx \E^o[Y \mid x, s]$.

\section{Learning composite endpoints}
\label{sec:method}

To challenge the limitations of existing strategies, we propose learning composite plug-in surrogate endpoints by searching for a function $f(S)$ with \emph{conditional} trial treatment effect most similar to the primary outcome, by solving
\begin{equation}\label{eq:main_obj}
    \underset{f \in \cF}{\mbox{minimize}} \;\; R_\tau^e(f) \coloneqq \E^e_{X}\left[\left(\tau_Y^e(X) - \tau^e_{f(S)}(X) \right)^2 \right]~.%
\end{equation}
The function $f$ will be learned \emph{before} the experiment starts from two sources of retrospective, typically observational, data sampled from distributions $p^o(X, T, S)$ and $p^o(X, S, Y)$, respectively. We do \emph{not} assume that all four variables are observed together. For instance, the first data could be the second-phase trials for a new treatment, while the latter could come from health records of past patients.

An endpoint $f(S)$ from an interpretable model family $\cF$ with small trial CATE error $R_\tau^e(f)$  satisfies our desiderata by optimizing for the predictability of the effect of $T$ on $Y$ (\ref{des:effect}--\ref{des:pred}), and serving as a plug-in surrogate based only on post-treatment variables, learned completely before the experiment (\ref{des:int}). Additionally, it supports \textbf{personalization} by being trained to capture treatment effect heterogeneity, and \textbf{transportability} by optimizing for use in $p^e(X)$ directly. 

Next, we prove the identifiability of $R_\tau^e(f)$, give strategies to minimize it, and discuss limitations of plug-in surrogates, paying attention to the conditions under which the method yields unbiased estimates of the primary treatment effect.

\subsection{Identifying the trial CATE error}
To solve \eqref{eq:main_obj} before the experiment starts, we must rely on statistical parameters estimable from observational data. Both $\tau_Y(x)$ and $\tau_{f(S)}(x)$ are unknown interventional quantities which must be causally \emph{identified} to avoid confounding and other biases. For this, we make classical identifying assumptions on the support and causal (in)dependences of variables, analogous to~\citet{athey2025surrogate}.
\begin{thmasmp}[Ignorability]\label{asmp:id} The context $X$ gives conditional exchangeability for $S$ and $Y$ in $P=o$, $\forall t \in \{0,1\}$,
$$
 Y(t) \indep T \mid X, P=o \quad S(t) \indep T \mid X, P=o~,
$$
and satisfies positivity (common support), 
$$
\forall x: p^o(x) > 0 \Longrightarrow p^o(T=t \mid X=x) > 0~.
$$
Endpoints satisfy consistency, $Y = Y(T), S=S(T)$.
\end{thmasmp}
Exchangeability can be guaranteed by graphical criteria, such as the backdoor criterion~\citep{pearl2009causality}, applied to the causal graph of the problem's variables. The four causal graphs in Figure~\ref{fig:cdags4} all guarantee exchangeability. We return to these in Section~\ref{sec:limitations} to discuss the biases of plug-in surrogates. For more discussions on graphical criteria for surrogacy, see~\citet{lauritzen2003graphical,vanderweele2013surrogate}.

\begin{figure}[t]
\centering
\begin{subfigure}{0.49\linewidth}
    \centering
    \includegraphics[width=\linewidth]{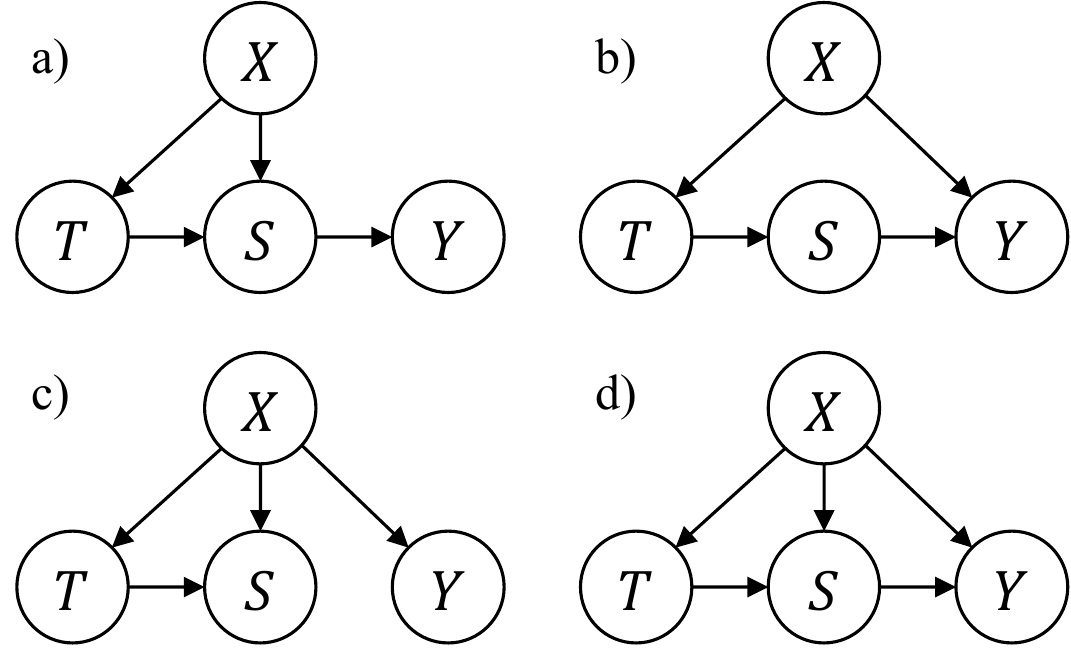}
     \caption{Special cases of causal graphs compatible with Assump.~\ref{asmp:id}--\ref{asmp:comp}, for which the CATE error \eqref{eq:main_obj} is identifiable.}
     \label{fig:cdags4}
\end{subfigure}
\hfill
\begin{subfigure}{0.49\linewidth}
    \centering
    \includegraphics[width=\linewidth]{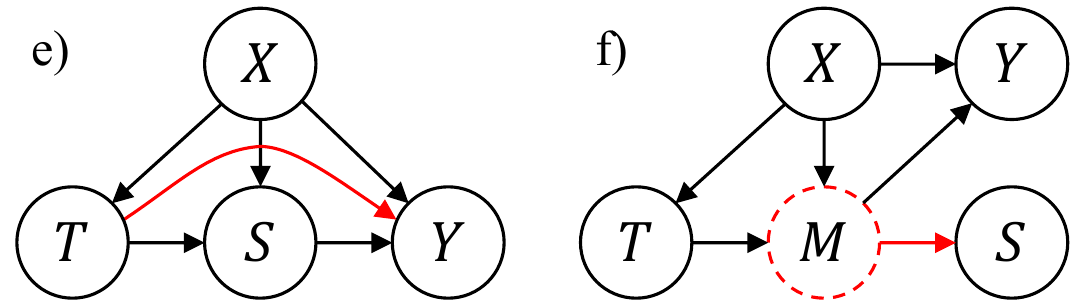}
     \caption{Challenging cases for surrogacy. In case e), there is a direct $T\rightarrow Y$ effect which prevents finding an optimal plug-in surrogate. In case f), the true mediator $M$ is unobserved, and $S$ are only proxies for $M$.}
     \label{fig:cdags2}
\end{subfigure}
\end{figure}

$S$ is a set of variables, each of which may be associated with a different set of confounders. Thus, justifying conditional exchangeability of $S(t)$ is, in general, more difficult than for a single outcome, $Y$. Moreover, performing an experiment intervening on $T$ can only remove confounding from the association between $T$ and $Y$, but not between $S$ and $Y$, and $S$ may not be possible to intervene upon itself. 

Under Assumption~\ref{asmp:id}, standard arguments yield that
\begin{align*}
\mu_Y^o(x,t) & \coloneqq  \E^o[Y(t) \mid X=x] = \E^o[Y \mid X=x, T=t] \\
& =  \E_S^o[\E_Y^o[Y \mid S, X=x, T=t] \mid X=x, T=t]
\end{align*}
and $p^o(S(t) \mid X=x) = p^o(S \mid X=x, T=t)$. 
For most definitions of surrogacy~\citep{athey2025surrogate,chen2007criteria,prentice1989surrogate}, it is additionally assumed that 
\begin{thmasmp}[Prentice surrogacy]\label{asmp:med}
    The surrogate set $S$ and context $X$ render the outcome independent of treatment,
    $$Y \indep T \mid X, S, P=e~.$$%
\end{thmasmp}%
Assumption~\ref{asmp:med} is not trivial to satisfy as it rules out $T$ having a direct effect on $Y$, or an effect mediated by an unobserved variable (see Figure~\ref{fig:cdags2}). It is \emph{more} plausible when $S$ is a set of multiple relevant variables, rather than a single factor, but still requires problem-specific justification.

Under Assumptions~\ref{asmp:id}--\ref{asmp:med}, the conditional expected potential outcomes $\mu_Y^o(x,t)$ can be identified from $p^o(X,T,S)$ and $p^o(X, S, Y)$ since the conditioning on $T$ in the inner expectation above can be dropped. 
It follows trivially that $\tau_{f(S)}^o(x)$ and $\tau_{Y}^o(x)$ are identified as well. 
To identify Objective \eqref{eq:main_obj}, we additionally need to relate $p^e$ and $p^o$. Broadly, this is a problem of transportability~\citep{pearl2011transportability}, typically solved by imposing assumptions on shared conditionals between $p^e$ and $p^o$, so also here. 
\begin{thmasmp}[Transportability]\label{asmp:comp}
    For the experiment and observational populations. $p^e$, $p^o$, it holds for all $x, t, s$ that
    \begin{align*}
    p^e(S \mid X=x, T=t) & = p^o(S \mid X=x, T=t)\\
    p^e(Y \mid X=x, S=s) & = p^o(Y \mid X=x, S=s)~.
    \end{align*}%
\end{thmasmp}%
The second equality is assumed by \citet{athey2025surrogate} as ``comparability''. Here, the assumption on the distribution of $S$ is made to support learning $f(S)$ ahead of the experiment. 

\begin{thmprop}\label{prop:cate}
    Under Assump.~\ref{asmp:id}--\ref{asmp:comp}, with $h^o(x,s)$ in \eqref{eq:surrogate_index},
    \begin{align*}
    \tau_{Y}^e(x) = \tau_Y^o(x) & = \E_S^o[h^o(x,S) \mid X=x, T=1] \\
    & - \E_S^o[h^o(x,S) \mid X=x, T=0] \\
    \tau_{f(S)}^e(x) = \tau_{f(S)}^o(x) & = \E_S^o[f(S) \mid X=x, T=1] \\
    & - \E_S^o[f(S) \mid X=x, T=0]~.
    \end{align*}
\end{thmprop}
We give a proof in \Cref{app:proofs} following from standard arguments. To fully identify~\eqref{eq:main_obj}, the trial population $p^e(X)$ or the density ratio $p^e(X)/p^o(X)$ must be known, to account for differences the between observational and experiment cohorts. The surrogate index strategy of  \citet{athey2025surrogate}, does not rely on pre-experiment access to this information since only the surrogate index $h^o(X,S)$ is transferred from the observational study. This is not an option under desideratum \ref{des:int}. 
Moving forward, we assume that $p^e(X)/p^o(X)$ is known a priori to an accuracy sufficient to find an approximate minimizer of \eqref{eq:main_obj}. This is possible, for example, when the main distributional difference between experimental and observational cohorts is due to inclusion criteria $I$ applied for the experiment, as is standard practice in clinical trials. With $\mathds{1}[x \in I]$ indicating that $x$ satisfies the inclusion criteria, it is plausible that $p^e(x) \propto p^o(x)\mathds{1}[x \in I]$.

Finally, building on Proposition~\ref{prop:cate}, under Assumptions~\ref{asmp:id}--\ref{asmp:comp}, we can identify the trial CATE error as follows
\begin{align}
R_\tau^e(f) &= \E^o_X \bigg[ \frac{p_e(X)}{p_o(X)}\big(\E^o_S[h^o(X, S) -f(S) \mid X, T=1] \nonumber \\
& \qquad\; - \E^o_S[h^o(X, S) -f(S) \mid X, T=0] \big)^2 \bigg]. \label{eq:risk_id}
\end{align}
which is estimable from $p^o(X,T,S)$ and $p^o(X,S,Y)$. Moving forward, we drop the superscript from $h^o(x,s)$.

\subsection{Minimizing the expected CATE error}
Minimizing \eqref{eq:risk_id} requires estimating $R^e_\tau(f)$ in a way that is amenable to optimization of $f$. Below, we present two strategies for this. Both start by fitting an estimate $\hat{h}(x,s) \approx \E^o[Y \mid X=x, S=s]$ by regression on an observational data set $D_y^o = \{(x_j', s_j', y_j')\}_{j=1}^{m_y}$ drawn from $p^o(X, S, Y)$.

\subsubsection*{Estimation through surrogate sampling}
Using $\hat{h}(x,s)$ in place of $h(x,s)$ in \eqref{eq:risk_id}, the expression depends only on variables $(x,s,t)$ and can be estimated using data set $D_t^o = \{(x_j, t_j, s_j)\}_{j=1}^{m_t}$. However, each sample holds only an observation of the factual potential outcome $S_i(t_i)$ for a given $x$, but not of the counterfactual $S_i(1-t_i)$, making it difficult to estimate the inner expectations. 

If an accurate model $\hat{p}(S \mid x, t) \approx p(S(t) \mid x)$  is available, the inner expectations of \eqref{eq:risk_id} can, in principle, be computed exactly, but only for trivially small problems or simple distributions (such as when $f, h$ are linear, see Section~\ref{sec:limitations}). 
Instead, we propose sampling potential outcomes $S(0)$ and $S(1)$ from $\hat{p}(S \mid x, t)$ to yield augmented data points $(x_i, s_i, t_i, \hat{s}(0)^1_i, ..., \hat{s}(0)^L_i, \hat{s}(1)^1_i, ..., \hat{s}(1)^L_i,)$,  with $L$ samples $\hat{s}(t)^l_i$ of each potential outcome, and used to estimate 
$$
\E_S[\hat{h}(x_i, S) \mid X=x_i, T=t] \approx \frac{1}{L}\sum_{l=1}^L \hat{h}(x_i, \hat{s}(t)_i^l),
$$
and $\E_S[f(S)\mid x_i, t]$, analogously. Plugging these estimates into \eqref{eq:risk_id} for $t \in \{0,1\}$, and averaging over $p^e(x)$, we have an estimator $\hat{R}^e_\tau(f)$ of the full trial CATE risk. We refer to the algorithm that fits $f(S)$ by minimizing this estimate as \algsamp{}, described in Algorithm~\ref{alg:methods}. In the stated version, we ignore population differences $p^e(x) \neq p^o(x)$. Such differences can be accommodated by re-weighting the risk average with the density ratio $p^e(x)/p^o(x)$.

The minimizers $f^*$ of \eqref{eq:risk_id} are translation invariant, that is, $R^e_\tau(f) = R^e_\tau(f+b)$ for any constant $b$, and may not align in scale with the primary outcome, $Y$. This can be easily post-calibrated using samples of $Y$ or the fitted function $\hat{h}$.

Estimating conditional densities $p(S \mid x, t)$ for continuous $S$ can be difficult and relies on the availability of large data sets~\citep{papamakarios2021normalizing}. Discretizing the variable set, the conditional density may be represented by a probability table, but results in expensive computation and statistical challenges. For example, when $S$ is composed of only six covariates, approximating each covariate by ten bins results in estimating $p(S=s|X=x,T=t)$ for $10^6$ different configurations $s$, for each value of $x$.

\begin{algorithm}[t]
\caption{Learning plug-in composite endpoints}
\label{alg:methods}
\small
\begin{algorithmic}[1]
\State \textbf{Input:} Observational treatment data $D^o_t = \{(x_j, t_j, s_j)\}_{j=1}^{m_t}$ and outcome data $D^o_y = \{(s_i', x_i', y_i')\}_{i=1}^{m_y}$, model class $\cF$
\State Fit outcome model $\hat{h}(x,s) \approx \E[Y \mid x, s]$ to $D^o_y$
\If{\algsamp{}}
    \State Fit surrogate models $\hat{g}(x, t) \approx p(S \mid x, t)$ for $t\in \{0,1\}$
    \For{$j=1, ..., m$, $l=1, ..., L$}
        \State Draw $\hat{s}(t)^l_j \sim \hat{g}(x_j, t)$ for $t\in \{0,1\}$
    \EndFor
    \State \resizebox{7.2cm}{!}{$\displaystyle\hat{f} = \argmin_{f\in \cF} \sum_{j=1}^{m_t} \Big(\sum_{t=0}^1 (-1)^t\sum_{l=1}^L [\hat{h}(x_j,\hat{s}(t)_j^l) - f(\hat{s}(t)_j^l)] \Big)^2$}
\ElsIf{\algbound{}}
    \State Fit propensity score $\hat{e}(x) \approx p(T=1 \mid x)$ to $D^o_t$
    \State Fit surrogate score $\hat{\rho}(x,s) \approx p(T=1 \mid x, s)$ to $D^o_t$
    \State $\hat{w}^2(x,s) = \big(\hat{\rho}(x,s)/\hat{e}(x) - (1-\hat{\rho}(x,s))/(1-\hat{e}(x))\big)^2$ 
    \State $\displaystyle \hat{f} = \argmin_{f\in \cF} \sum_{j=1}^{m_t}\hat{w}^2(x_j,s_j)(\hat{h}(x_j,s_j) - f(s_j))^2$
\EndIf
\State \Return $\hat{f}$
\State {\scriptsize \emph{Note: If $p^e(x) \neq p^o(x)$, the objectives should be weighted by $p^e(x) / p^o(x)$.}}
\end{algorithmic}
\end{algorithm}

\subsubsection*{Minimizing an upper bound on the CATE risk}
To avoid sampling-based estimation or costly numerical integration of Objective~\ref{eq:main_obj}, we can minimize an upper bound of the objective, resulting in a weighted regression problem. 
\begin{thmprop}\label{prop:risk_bound}
    For any surrogate function $f$, under Assumptions~\ref{asmp:id}--\ref{asmp:comp}, the trial CATE risk is bounded as  
    \begin{align}
    R^o_\tau(f) &\leq \E_{X,S}[w^2(X,S)(h(X,S) - f(S))^2] \label{eq:cate_bound} \\
    & \leq \E_{X,S}[w^2(X,S)(Y - f(S))^2] + \sigma^2, \nonumber
    \end{align}
    where $w^2(x,s) = \left(\frac{p(T=1 \mid s,x)}{p(T=1 \mid x)} - \frac{p(T=0 \mid s,x)}{p(T=0 \mid x)}\right)^2$. The second inequality holds if the outcome has exogenous, centered, additive noise: $Y = h(X,S)+\epsilon$ where $\E[\epsilon]=0$, $\mbox{Var}(\epsilon) \leq \sigma^2$ and $\epsilon \indep S, X, T$. If $p_o(X)/p_e(X)$ is known, using weights $\tilde{w}^2(x,s) = \frac{p_o(x)}{p_e(x)}w^2(x,s)$ yields a bound on $R^e_\tau(f)$.
\end{thmprop}
A proof is given in Appendix~\ref{app:risk_bound}. There, we also derive analogous results to \eqref{eq:cate_bound} with non-squared,  additive weights $w^+(x,s) = \frac{p(T=1 \mid s,x)}{p(T=1 \mid x)} + \frac{p(T=0 \mid s,x)}{p(T=0 \mid x)}$ for the squared loss and absolute difference weights $w^1(x,s) = \left|\frac{p(T=1 \mid s,x)}{p(T=1 \mid x)} - \frac{p(T=0 \mid s,x)}{p(T=0 \mid x)} \right|$ for the $L_1$ loss. A downside of the $w^+$ weights is that they are uniform and uninformative when $p(T=1 \mid x) \equiv 0.5$, e.g., in a randomized experiment.

\textbf{Remarks.\;\;} Minimizing \eqref{eq:cate_bound} amounts to solving a weighted regression problem with $x, s$ as input and $h(x, s)$, or $y$, as label.\footnote{However, using $Y$ as label requires access to samples of all variables $(x,t,s,y)$ at once, which may not be available.} The weight $w^2(x,s)$ places emphasis on points $(x,s)$ where $S$ is informative of $T$ conditioned on $X$, which is precisely points $(x,s)$ where $s$ is more affected by the treatment assignment. This stands in contrast to standard outcome-on-surrogate regression (with weights $w(x,s) \equiv 1$), which has no such preference. 
An algorithm based on minimizing the bound in \eqref{eq:cate_bound} using an estimate $\hat{h}$ in place of $h$ is presented in Algorithm~\ref{alg:methods} as ``\algbound{}''.

\subsection{The biases of plug-in surrogates}
\label{sec:limitations}

It is instructive to ask: When is there \emph{any} plug-in surrogate $f(S)$ such that $\tau_f(x) = \tau_Y(x)$ ?\footnote{We can drop the population superscripts here as $\tau^e(x) = \tau^o(x)$ under Assumptions~\ref{asmp:id}--\ref{asmp:comp}.} And if there is one, can it be learned using one of our methods? If $\|\tau_S\|>0$, as described previously, it is always possible to find $f$ such that the ATEs are matched, $\tau_{f(S)}=\tau_Y$, but do our methods return such an $f$? To begin to answer this, we examine the four cases in Figure~\ref{fig:cdags4} where $R^e_\tau(f)$ is identifiable. For proofs and discussions on each case, see Appendix~\ref{app:examples}.

\paragraph{Case a) Confounded surrogate mediator.}\mbox{}\\
Here, $Y \indep X \mid S$, and so $h(x,s) = \E[Y \mid X=x, S=s]$ $= \E[Y\mid S=s]$. Consequently, $\tau_Y(x) = \E_S[\E[Y \mid S], X=x, T=1] -\E_S[\E[Y \mid S], X=x, T=0] = \tau_{f(S)}(x)$ for a perfect regression $f(s) = \E[Y \mid S=s]$. This $f$ both minimizes \eqref{eq:main_obj} and yields an unbiased estimate of $\tau_Y(x)$ and $\tau_Y$. Moreover, \emph{any} weighted regression, with full weight support over $x, s$ and sufficient functional capacity, will learn this function, including the weights $w^2(x,s)$. 

\paragraph{Case b) Confounded outcome, surrogate mediator.}\mbox{}\\
Here, $S \indep X \mid T$ and $S$ carries no information about effect heterogeneity due to $X$. Out of the four cases, this is the worst for \emph{any} plug-in surrogate in terms of matching CATE, whether found by minimization of \eqref{eq:main_obj} or not. In general, for all $f$, $\exists x,s : f(s) \neq h(x, s)$ and, unless $p^o(s\mid T=1, x)=p^o(s \mid T=0, x)$, implying that $S$ is unaffected by $T$, $\exists x : \tau_{f(S)}(x) \neq \tau_Y(x)$. Thus, plug-in surrogate estimates of the primary-outcome CATE are generally biased. It is however, possible, to match the ATE, and bound regression using the weights $w^{+}$ or $w^1$ do yield $\tau_{f(S)} = \tau_Y$.

\paragraph{Case c) Confounded surrogate, no effect.}\mbox{}\\
Here, the primary-outcome effects $\tau_Y(x) = \tau_Y = 0$. Any function $f$ such that $f(S) \indep T \mid X$ will yield $\tau_{f(S)}(x) = 0$, such as a constant $f(S) = c$. Thus, an unbiased plug-in surrogate exists and would be found by surrogate sampling, as it minimizes \eqref{eq:main_obj}, but is not guaranteed to be found by bound regression or outcome regression since $Y \not\indep T \mid S$ due to confounding by $X$.

\paragraph{Case d) Confounded surrogate and outcome.}\mbox{}\\
Here, the extent to which $\tau_{Y}$ can be matched by a function $f(S)$ is determined by several factors, such as the functional form of $\E[Y \mid s, x]$ and the extent to which information in $x$ that alters $\tau_Y(x)$ is captured by $s$. In general, no plug-in surrogate is unbiased. We give the regression weights unbiased for the ATE in Appendix~\ref{app:examples_d}, but these are potentially negative and are not suitable for weighted regression.

\paragraph{The linear case.} When the functions $f$ and $h$ are both linear with respect to $s$, that is, $h(x, s)  = h_x(x) + \beta_h^\top s + b_h$ and $f(s) =\beta_f^\top s + b_f$, the trial CATE risk simplifies to 
\begin{align*}
R_\tau^e(f) = \E^e_X[((\beta_h-\beta_f)^\top \tau_S(x))^2],
\end{align*}
as shown in Appendix~\ref{app:linear_risk}. Taking $\beta_f =\beta_h$ yields $R_\tau^e(f)=0$, and hence, the plug-in surrogate $f(S)$ that minimizes \eqref{eq:main_obj} is unbiased for both $\tau_Y(x)$ and $\tau_Y$, as is the surrogate index estimator~\cite{athey2025surrogate}. However, standard (unweighted) outcome regression with a linear model $\theta^\top s$ is biased in general, since $\E[Y|S=s] = \E[h_x(X) \mid s] + \beta_h^\top s + b_h$ and $\E_{S(1)}[\E_X[h_x(X) \mid S(1)]] \neq \E_{S(1)}[\E_X[h_x(X) \mid S(0)]]$ without further assumptions.

\paragraph{Choosing moderators $X$ to personalize policies.} Based on the four cases above, the best-case scenario to learn a plug-in surrogate $f(S)$ is when $X$ and $S$ are tightly linked, to the point that $Y \indep X \mid S$. We can design successful surrogate studies by \emph{choosing} both $X$ and $S$ to achieve this. For example, by letting $X$ include pre-treatment measurements of the variables in $S$, and $S$ follow-up measurements of the confounders in $X$ makes it likely that the information in $X$ relevant to predict the effect on the outcome $Y$ is contained in $S$. We remind the reader that $X$ is composed of two types of variables: (i) adjustment variables to remove confounding, (ii) effect moderators used to personalize decisions. We treat them as the same set here, but they need not be. If a confounder is unhelpful for personalization, it need not be included in the CATE conditioning set, and it removes one more source of potential effect variance for $S$ to explain. 

\section{Experiments}
We evaluate our proposed methods in comparison with standard baselines and recent methods for surrogate learning in a suite of synthetic experiments and on data from the IHDP randomized controlled trial~\citep{martin2008long}. For details on the experimental setup, see Appendix~\ref{app:experiment_details}.

\begin{table*}[t]
    \centering
    \caption{Synthetic scenarios a--d), matching the graphs in Figure~\ref{fig:cdags4}, see Appendix~\ref{app:experiment_details}. Results are averaged over linear and nonlinear scenarios within cases. We report $95\%$ confidence intervals computed using the standard error across runs. In Cases b--c), the CATE is constant and $R^2$ uninformative. $^\dagger$The surrogate index depends directly on pre-treatment variables.}
    \label{tab:synth_a_e}
    \setlength{\tabcolsep}{3 pt}
\begin{tabular}{lcccccccccc}
    \toprule
     Composite Scenarios & \multicolumn{2}{c}{\footnotesize \makecell{\textbf{Case a)}\\ $\overline{ATE}$=3.74}} & \multicolumn{1}{c}{\footnotesize  \makecell{\textbf{Case b)}\\ $\overline{ATE}$=2.12}} & \multicolumn{1}{c}{\footnotesize  \makecell{\textbf{Case c)}\\ $\overline{ATE}$=0}} & \multicolumn{2}{c}{\footnotesize  \makecell{\textbf{Case d)}\\ $\overline{ATE}$=5.67}} & \multicolumn{2}{c}{\footnotesize \makecell{\textbf{Case d) Linear}\\ $\overline{ATE}$=18.54}} \\
    \textbf{Method} & MAE $\hat{\tau}\downarrow$ & {\footnotesize $R^2 \hat{\tau}(x)\uparrow$} & {\footnotesize MAE $\hat{\tau}\downarrow$} & {\footnotesize MAE $\hat{\tau}\downarrow$} & 
    {\footnotesize MAE $\hat{\tau}\downarrow$} & {\footnotesize $R^2 \hat{\tau}(x)\uparrow $} & {\footnotesize MAE $\hat{\tau}\downarrow$ }& {\footnotesize $R^2 \hat{\tau}(x)\uparrow$} \\
    \midrule
    Reg.-sel.-reg. (lin) & $2.22\pm$ {\scriptsize 0.16} & $0.27  \pm$ {\scriptsize 0.23} & $0.72 \pm${\scriptsize 0.02} & $0.32\pm$ {\scriptsize 0.01} & $2.35\pm$ {\scriptsize 0.13} & $-24.0  \pm$ {\scriptsize 16.0} & $1.98\pm$ {\scriptsize 0.06} & $0.85  \pm$ {\scriptsize 0.01} \\
    Reg.-sel.-reg. (tree) & $1.23\pm$ {\scriptsize 0.15} & $0.41  \pm$ {\scriptsize 0.25} & $0.71\pm${\scriptsize 0.07} &  $0.26\pm$ {\scriptsize 0.01} & $1.53\pm$ {\scriptsize 0.13} & $-23.9  \pm$ {\scriptsize 2.53} & $1.40\pm$ {\scriptsize 0.03} & $0.81  \pm$ {\scriptsize 0.01} \\
    Outcome reg. (lin) & $1.53\pm$ {\scriptsize 0.15} & $0.40\pm$ {\scriptsize 0.02} & $0.14\pm${\scriptsize 0.06} &  $0.55\pm${\scriptsize 0.02} & $1.22\pm${\scriptsize 0.12} & $0.51 \pm$ {\scriptsize 0.15} & $1.63\pm$ {\scriptsize 0.05} & $0.90  \pm$ {\scriptsize 0.00} \\
    Outcome reg. (tree) & $0.87\pm$ {\scriptsize 0.14} & $0.33\pm$ {\scriptsize 0.03} & $0.40 \pm${\scriptsize 0.06} &  $0.41\pm${\scriptsize 0.03} & $0.66\pm${\scriptsize 0.12} & $0.40  \pm$ {\scriptsize 0.18} & $0.93\pm$ {\scriptsize 0.07}  & $0.89  \pm$ {\scriptsize 0.01} \\
    Surrogate index$^\dagger$ (lin) & $1.68\pm$ {\scriptsize 0.15} & $0.31 \pm$ {\scriptsize 0.03}  & $0.16\pm${\scriptsize 0.06} &  $0.64\pm$ {\scriptsize 0.03} & $1.17\pm${\scriptsize 0.13} & $0.51  \pm$ {\scriptsize 0.13} & $0.04\pm$ {\scriptsize 0.04} & $0.96  \pm$ {\scriptsize 0.00} \\
    Surrogate index$^\dagger$ (tree) & $0.86\pm$ {\scriptsize 0.14} & $0.33  \pm$ {\scriptsize 0.5} & $0.40 \pm${\scriptsize 0.06} &  $0.35\pm${\scriptsize 0.03} & $0.67\pm${\scriptsize 0.12} & $0.01  \pm$ {\scriptsize 0.40} & $0.16\pm$ {\scriptsize 0.08} & $0.90  \pm$ {\scriptsize 0.01} \\
    Surrogate index$^\dagger$ (hgb) & $0.34\pm$ {\scriptsize 0.12} & $0.54  \pm$ {\scriptsize 0.32} & $0.11\pm${\scriptsize 0.06} &  $0.25\pm${\scriptsize 0.03} & $0.21\pm${\scriptsize 0.11} & $0.61  \pm$ {\scriptsize 0.15} & $0.07\pm$ {\scriptsize 0.05} & $0.95  \pm$ {\scriptsize 0.00} \\
    \midrule
    Bound reg. (lin) & $1.69\pm$ {\scriptsize 0.15} & $0.37  \pm$ {\scriptsize 0.26} & $0.17\pm${\scriptsize 0.06} & $0.35\pm${\scriptsize 0.02} & $1.33\pm${\scriptsize 0.13} & $0.48  \pm$ {\scriptsize 0.14} & $1.11\pm$ {\scriptsize 0.05} & $0.93  \pm$ {\scriptsize 0.00} \\
    Bound reg. (tree) & $0.68\pm$ {\scriptsize 0.13} & $0.44  \pm$ {\scriptsize 0.29} & $0.31\pm${\scriptsize 0.06} & $0.32\pm${\scriptsize 0.02} & $0.81 \pm${\scriptsize 0.11} & $-2.10\pm$ {\scriptsize 1.85} & $1.27\pm$ {\scriptsize 0.06} &$0.89  \pm$ {\scriptsize 0.01} \\
    Surrogate sampl. (lin) & $0.50\pm$ {\scriptsize 1.63} & $0.65  \pm$ {\scriptsize 0.27} & $0.09\pm${\scriptsize 0.05} & $0.20\pm${\scriptsize 0.01} & $0.40 \pm${\scriptsize 0.12} & $0.66  \pm$ {\scriptsize 0.15} & $0.03\pm$ {\scriptsize 0.03} & $0.97  \pm$ {\scriptsize 0.00} \\
    \bottomrule
\end{tabular}
\end{table*}

\paragraph{Methods.} As baselines, we include outcome regressions, estimating $\E[Y\mid S=s]$, as well a single-surrogate regression (Reg-Sel-Reg), learned by fitting a regression, selecting the variable with the largest absolute coefficient, and regressing only on that. We also include the surrogate index $h(x,s) =\hat{\E}[Y|x, s]$ as a comparison that depends directly on pre-treatment variables, $x$. All baselines come in linear and tree model variants, with gradient boosting regression for the surrogate index. For nonlinear models, we do a $5$-fold cross validation and grid search to select hyperparameters.

We implement \algbound{} using weighted linear regression. For fitting propensity weights $e(x)$ and surrogate score $\rho(x,s)$, we use logistic regression, fitted using a $L_2$ regularization $\alpha=1$. For \algsamp{}, we first fit nuisance random forest models for $\hat{\E}[S|x,T=t]$ and sample $L=50$ residuals of the fitted model via bootstrapping to simulate the conditional $\hat{p}(S|x, T=t)$. Then, we fit a LASSO regression model with regularization strength $\lambda=0.01$ to surrogates $\hat{s}(0), \hat{s}(1)$ sampled from the conditional (see Algorithm~\ref{alg:methods}), using PyTorch~\cite{pytorch}.
For IHDP, we additionally implement a tree variant of \algbound{} using the \texttt{PyDL8.5} library \citep{aglin2021pydl8}, which learns globally optimal trees, using our proposed weighted squared-error objective. This variant is included to explore interpretable composite surrogates in real-data applications.

For each surrogate $f$, we estimate the trial surrogate ATE $\tau^e_{f(S)}$ using the  average treatment-group outcomes in the trial $\hat{\E}^e[f(S) \mid T=1 ]-\hat{\E}^e[f(S) \mid T=0]$. We estimate the surrogate CATE $\tau^e_{f(S)}(x)$, using a T-learner, fitting linear models $g_t(x) \approx \E_S[f(S) |x, t]$ for $t\in \{0,1\}$ and defining $\tau^e_{f(S)} = g_1(x) - g_0(x)$. We compare this to an estimate using the simulated potential outcomes $s(1)$ and $s(0)$ and found strong agreement in results (see appendix Figure~\ref{appfig:r2_comp}). 

\paragraph{Evaluation metrics.}
We evaluate the quality of learned surrogates $f(S)$ by the mean absolute error $(\operatorname{MAE}~\hat{\tau})$ of surrogate-based estimates $\hat{\tau}^e_{f(S)}$ with respect to the primary-outcome ATE $\tau_Y^e$, and the coefficient of determination $R^2$ of the surrogate CATE $\hat{\tau}^e_{f(S)}(x)$ with respect to $\tau_Y^e(x)$, averaged across seeds and scenarios. In synthetic data, we use ATE and CATE computed using simulated potential outcomes as ground truth. In IHDP, we use difference-in-means estimates of the ATE as ground-truth. The uncertainty of each metric is assessed using 95\% percentile confidence intervals computed using bootstrap resampling. 

\subsection{Synthetic experiments}
\label{sec:synthetic}
We generate synthetic observational and trial cohorts from structural causal models over $(X,T,S,Y)$, obeying scenarios a-d) in Figure~\ref{fig:cdags4},  where $T\in\{0,1\}$, $X \in \bbR^d$ for $d\in \{2,5\}$ and the surrogate set $S$ is composed of three variables, $S_{\med} \in \bbR^3$, $S_{\leaf} \in \bbR^2$, and $S_{\proxy} \in \bbR^2$ where $S_{\med}, S_{\leaf}$ are affected by $T$ and $S_{\med}, S_{\proxy}$ are causes of $Y$. In the observational cohort, treatments are assigned based on $X$, but randomized uniformly in the experimental cohort; otherwise data generation follows the same process. Across scenarios, we include linear and nonlinear outcome mappings and variations in treatment overlap. Each scenario is repeated with multiple random seeds.
For a full description, including structural equations, see Appendix~\ref{app:experiment_details}. The results of the synthetic experiments are presented in Table~\ref{tab:synth_a_e}.

\textbf{Surrogate sampling achieved the best overall results.} 
Across scenarios and metrics, directly minimizing the estimated CATE trial error through surrogate sampling (Algorithm~\ref{alg:methods}) with a linear model achieved the best results for both ATE and CATE among methods that return a plug-in surrogate function of only post-treatment variables, nearly matching the performance of the strongest baseline, a surrogate index using a more flexible histogram gradient boosting model which uses both $X$ and $S$ as inputs.

\textbf{Bound regression comparable to outcome regression.} In most cases, we see little benefit of the bound regression algorithm, with Case d) Linear as a notable exception. Here, the linear bound regression outperforms the (unweighted) linear outcome regression. This is expected as there exists an unbiased linear plug-in surrogate in this case, as demonstrated by the Surrogate sampling method. 

\textbf{Composite surrogates outperform single-variable surrogates}. Across scenarios, seeds, and model classes, the Regress-select-regression baselines fared the worst, indicating that a single-variable surrogate is insufficient to predict the effect of the intervention on the primary outcome.

\subsection{Real-world experiment: IHDP}
\label{sec:ihdp}

The Infant Health and Development Program (IHDP)\footnote{A semi-synthetic version of the IHDP study~\citep{hill2011bayesian} has been used widely in the machine learning community, but this is \emph{not} used here. We use the original data from the randomized trial.} was a randomized longitudinal multi-site trial of early childhood intervention for low birth weight, premature infants~\citep{10.1001/jama.1990.03440220059030}. Baseline covariates $X$ include information about the mother and child measured at birth, and the surrogate vector $S$ consists of post-baseline measures of child development, collected up to 24 months after birth. The primary outcome $Y$ is the Stanford--Binet IQ score at 36 months corrected age. This outcome is not possible to measure before 24 months, as children of this age cannot complete an IQ score~\citep{laurent1992review}, necessitating the use of a surrogate to estimate the effect at this time.

We restrict the cohort to complete cases on $\{T,Y,X,S\}$ and create a 70/30 split stratified by treatment status with a fixed random seed. The 70-split is used to learn surrogate models as the $P=o$ cohort; the held-out split is used exclusively for trial validation, $P=e$. Table~\ref{tab:ihdp_ate_fixed} reports ATE recovery on the held-out split for the fixed-propensity specification. We show the experimental ATE, the model-implied ATE, and the absolute ATE error. Variable definitions and measurement details are provided in Appendix~\ref{app:ihdp_variables}

\begin{table}[t]
\centering
\caption{ATE recovery on IHDP. The ground-truth ATE, $\tau^e = 6.47$.  ATE and MAE are reported as bootstrap means with 95\% percentile CIs.
$^\dagger$The surrogate index depends directly on pre-treatment variables. }
\label{tab:ihdp_ate_fixed}
\setlength{\tabcolsep}{4pt}
\begin{tabular}{lcc}
\toprule
Method & $\widehat{\text{ATE}}$ (95\% CI) & MAE (95\% CI) \\
\midrule
Reg-Sel-Reg (lin) 
& 6.11 {\scriptsize (2.20, 10.34)} 
& 1.78 {\scriptsize (0.13, 4.65)} \\

Reg-Sel-Reg (tree) 
& 7.55 {\scriptsize (3.28, 11.34)} 
& 1.92 {\scriptsize (0.09, 4.95)} \\

Outcome reg. (lin) 
& 5.41 {\scriptsize (0.89, 9.37)}  
& 1.98 {\scriptsize (0.09, 5.68)} \\

Outcome reg. (tree) 
& 7.82 {\scriptsize (3.56, 11.80)} 
& 2.05 {\scriptsize (0.14, 5.43)} \\

Surrogate index$^\dagger$ (lin) 
& 4.79 {\scriptsize (-0.06, 9.59)}  
& 2.32 {\scriptsize (0.07, 6.53)} \\

Surrogate index$^\dagger$ (tree) 
& 6.27 {\scriptsize (2.15, 10.35)} 
& 1.70 {\scriptsize (0.08, 4.91)} \\

\midrule
Bound reg. (lin)     
& 4.69 {\scriptsize (0.91, 8.48)}  
& 2.13 {\scriptsize (0.12, 5.56)} \\

Bound reg. (tree)       
& 6.74 {\scriptsize (2.94, 9.90)} 
& 1.52 {\scriptsize (0.07, 3.69)} \\

Bound reg. (DL8.5)  
& 6.60 {\scriptsize (2.84, 10.24)} 
& 1.56 {\scriptsize (0.06, 4.14)} \\

Surrogate sampl. (lin)        
& 5.49 {\scriptsize (2.75, 8.03)}  
& 1.35 {\scriptsize (0.04, 3.72)} \\
\bottomrule
\end{tabular}
\end{table}

\textbf{Broad agreement on the benefits of the intervention.}
All methods estimate a positive overall treatment effect on IQ at 36 months, consistent with the experimental benchmark $\tau^e = 6.47$. Bootstrap means are generally close to the ground-truth ATE, though uncertainty remains substantial across specifications.  Most methods identify the {\tt Bayley\_MDI} score as a prominent surrogate.  Importantly, all methods would have led to the same positive conclusion about the treatment effect using data available at 24 months rather than waiting for the 36 month IQ outcome, highlighting the practical value of early surrogate information.

\textbf{Minimizing the trial CATE error yields the most predictive surrogates.}
Surrogate sampling achieves the smallest bootstrap MAE, followed by tree bound regression and DL8.5. These methods combine relatively small bias with tighter uncertainty bands compared to simpler regression baselines and the surrogate index. 
Selection and outcome-regression approaches display wider dispersion and moderately larger errors, while the linear surrogate index exhibits the largest deviation from the experimental ATE.

\section{Conclusion}
We have studied the learning of plug-in composite surrogate endpoints, functions of post-treatment variables that act as substitutes for unobserved long-term outcomes in randomized experiments. We proposed two algorithms based on minimizing estimates or bounds of the CATE error---the expected squared difference between the conditional average treatment effect estimated using the surrogate and using the primary outcome. In evaluations on multiple simulated scenarios and data from a real-world randomized experiment, we find that our method, based on directly modeling the surrogate effect by sampling surrogate counterfactuals, returns stronger plug-in endpoints than baselines.
In future work, we will explore a tree-based variant of surrogate sampling, which requires a new tree algorithm as the objective depends on both potential outcomes. This variant will benefit from the low bias of our best-performing method, and the capability to capture nonlinear treatment effect heterogeneity.
\newpage
\begin{ack}
This work was partially supported by the Wallenberg AI, Autonomous Systems and Software Program (WASP) funded by the Knut and Alice Wallenberg Foundation.

The computations and data handling were enabled by resources provided by the National Academic Infrastructure for Supercomputing in Sweden (NAISS), partially funded by the Swedish Research Council through grant agreement no. 2022-06725.
\end{ack}
\bibliography{references}

@article{chen2007criteria,
  title={Criteria for surrogate end points},
  author={Chen, Hua and Geng, Zhi and Jia, Jinzhu},
  journal={Journal of the Royal Statistical Society Series B: Statistical Methodology},
  volume={69},
  number={5},
  pages={919--932},
  year={2007},
  publisher={Oxford University Press}
}

@inproceedings{tripuraneni2024choosing,
  title={Choosing a proxy metric from past experiments},
  author={Tripuraneni, Nilesh and Richardson, Lee and D'Amour, Alexander and Soriano, Jacopo and Yadlowsky, Steve},
  booktitle={Proceedings of the 30th ACM SIGKDD Conference on Knowledge Discovery and Data Mining},
  pages={5803--5812},
  year={2024}
}

@article{frangakis2002principal,
  title={Principal stratification in causal inference},
  author={Frangakis, Constantine E and Rubin, Donald B},
  journal={Biometrics},
  volume={58},
  number={1},
  pages={21--29},
  year={2002},
  publisher={Oxford University Press}
}

@article{rubin1974estimating,
  title={Estimating causal effects of treatments in randomized and nonrandomized studies.},
  author={Rubin, Donald B},
  journal={Journal of educational Psychology},
  volume={66},
  number={5},
  pages={688},
  year={1974},
  publisher={American Psychological Association}
}

@article{freedman1992statistical,
  title={Statistical validation of intermediate endpoints for chronic diseases},
  author={Freedman, Laurence S and Graubard, Barry I and Schatzkin, Arthur},
  journal={Statistics in medicine},
  volume={11},
  number={2},
  pages={167--178},
  year={1992},
  publisher={Wiley Online Library}
}

@article{fleming_demets_1996surrogate,
  title={Surrogate end points in clinical trials: are we being misled?},
  author={Fleming, Thomas R and DeMets, David L},
  journal={Annals of internal medicine},
  volume={125},
  number={7},
  pages={605--613},
  year={1996},
  publisher={American College of Physicians}
}

@article{fleming1996surrogate,
  title={Surrogate endpoints in clinical trials},
  author={Fleming, Thomas R},
  journal={Drug Information Journal},
  volume={30},
  number={2},
  pages={545--551},
  year={1996},
  publisher={SAGE Publications Sage CA: Los Angeles, CA}
}

@phdthesis{meyvisch2020surrogate,
  title={Surrogate marker evaluation in clinical trials using methods of causal inference},
  author={Meyvisch, Paul},
  year={2020},
  school={KU Leuven}
}

@article{ciani2023framework,
  title={A framework for the definition and interpretation of the use of surrogate endpoints in interventional trials},
  author={Ciani, Oriana and Manyara, Anthony M and Davies, Philippa and Stewart, Derek and Weir, Christopher J and Young, Amber E and Blazeby, Jane and Butcher, Nancy J and Bujkiewicz, Sylwia and Chan, An-Wen and others},
  journal={EClinicalMedicine},
  volume={65},
  year={2023},
  publisher={Elsevier}
}

@incollection{molenberghs2024statistical,
  title={The Statistical Evaluation of Surrogate Endpoints in Clinical Trials},
  author={Molenberghs, Geert and Alonso Abad, Ariel and Van der Elst, Wim},
  booktitle={Biostatistics in Biopharmaceutical Research and Development: Clinical Trial Analysis, Volume 2},
  pages={243--286},
  year={2024},
  publisher={Springer}
}

@book{burzykowski2005evaluation,
  title={The evaluation of surrogate endpoints},
  author={Burzykowski, Tomasz and Buyse, Marc and Molenberghs, Geert},
  volume={427},
  year={2005},
  publisher={Springer}
}

@article{gilbert2006evaluating,
  title={Evaluating causal effect predictiveness of candidate surrogate endpoints},
  author={Gilbert, Peter B and Hudgens, Michael},
  year={2006},
  journal={Biometrics},
}

@inproceedings{wang2022surrogate,
  title={Surrogate for long-term user experience in recommender systems},
  author={Wang, Yuyan and Sharma, Mohit and Xu, Can and Badam, Sriraj and Sun, Qian and Richardson, Lee and Chung, Lisa and Chi, Ed H and Chen, Minmin},
  booktitle={Proceedings of the 28th ACM SIGKDD conference on knowledge discovery and data mining},
  pages={4100--4109},
  year={2022}
}

@article{temple1999surrogate,
  title={Are surrogate markers adequate to assess cardiovascular disease drugs?},
  author={Temple, Robert},
  journal={Jama},
  volume={282},
  number={8},
  pages={790--795},
  year={1999},
  publisher={American Medical Association}
}

@article{esposito2004effect,
  title={Effect of a Mediterranean-style diet on endothelial dysfunction and markers of vascular inflammation in the metabolic syndrome: a randomized trial},
  author={Esposito, Katherine and Marfella, Raffaele and Ciotola, Miryam and Di Palo, Carmen and Giugliano, Francesco and Giugliano, Giovanni and D'Armiento, Massimo and D'Andrea, Francesco and Giugliano, Dario},
  journal={Jama},
  volume={292},
  number={12},
  pages={1440--1446},
  year={2004},
  publisher={American Medical Association}
}

@article{martinez2015benefits,
  title={Benefits of the Mediterranean diet: insights from the PREDIMED study},
  author={Mart{\'\i}nez-Gonz{\'a}lez, Miguel A and Salas-Salvad{\'o}, Jordi and Estruch, Ram{\'o}n and Corella, Dolores and Fit{\'o}, Montse and Ros, Emilio and Predimed Investigators and others},
  journal={Progress in cardiovascular diseases},
  volume={58},
  number={1},
  pages={50--60},
  year={2015},
  publisher={Elsevier}
}

@article{chan2006apolipoproteins,
  title={Apolipoproteins as markers and managers of coronary risk},
  author={Chan, DC and Watts, GF},
  journal={Journal of the Association of Physicians},
  volume={99},
  number={5},
  pages={277--287},
  year={2006},
  publisher={Oxford University Press}
}

@article{keys1984seven,
  title={The seven countries study: 2,289 deaths in 15 years},
  author={Keys, Ancel and Menotti, Alessandro and Aravanis, Christ and Blackburn, Henry and Djordevi{\v{c}}, Bozidar S and Buzina, Ratko and Dontas, AS and Fidanza, Flaminio and Karvonen, Martti J and Kimura, Noboru and others},
  journal={Preventive medicine},
  volume={13},
  number={2},
  pages={141--154},
  year={1984},
  publisher={Elsevier}
}

@article{baker2003perfect,
  title={A perfect correlate does not a surrogate make},
  author={Baker, Stuart G and Kramer, Barnett S},
  journal={BMC medical research methodology},
  volume={3},
  number={1},
  pages={16},
  year={2003},
  publisher={Springer}
}

@article{yetley2017surrogate,
  title={Surrogate disease markers as substitutes for chronic disease outcomes in studies of diet and chronic disease relations},
  author={Yetley, Elizabeth A and DeMets, David L and Harlan Jr, William R},
  journal={The American journal of clinical nutrition},
  volume={106},
  number={5},
  pages={1175--1189},
  year={2017},
  publisher={Oxford University Press}
}

@article{elliott2015surrogacy,
  title={Surrogacy marker paradox measures in meta-analytic settings},
  author={Elliott, Michael R and Conlon, Anna SC and Li, Yun and Kaciroti, Nico and Taylor, Jeremy MG},
  journal={Biostatistics},
  volume={16},
  number={2},
  pages={400--412},
  year={2015},
  publisher={Oxford University Press}
}

@article{robins1992identifiability,
  title={Identifiability and exchangeability for direct and indirect effects},
  author={Robins, James M and Greenland, Sander},
  journal={Epidemiology},
  volume={3},
  number={2},
  pages={143--155},
  year={1992},
  publisher={LWW}
}

@article{lauritzen2003graphical,
  title={Graphical models for surrogates},
  author={Lauritzen, Steffen L},
  journal={Bull. Int. Statist. Inst},
  volume={60},
  pages={144--147},
  year={2003}
}

@inproceedings{pearl2011transportability,
  title={Transportability of causal and statistical relations: A formal approach},
  author={Pearl, Judea and Bareinboim, Elias},
  booktitle={Proceedings of the AAAI Conference on Artificial Intelligence},
  volume={25},
  pages={247--254},
  year={2011}
}

@book{pearl2009causality,
  title={Causality},
  author={Pearl, Judea},
  year={2009},
  publisher={Cambridge university press}
}

@article{joffe2009related,
  title={Related causal frameworks for surrogate outcomes},
  author={Joffe, Marshall M and Greene, Tom},
  journal={Biometrics},
  volume={65},
  number={2},
  pages={530--538},
  year={2009},
  publisher={Oxford University Press}
}

@article{vanderweele2013surrogate,
  title={Surrogate measures and consistent surrogates},
  author={VanderWeele, Tyler J},
  journal={Biometrics},
  volume={69},
  number={3},
  pages={561--565},
  year={2013},
  publisher={Wiley Online Library}
}

@article{zhang2023evaluating,
  title={Evaluating the surrogate index as a decision-making tool using 200 a/b tests at netflix},
  author={Zhang, Vickie and Zhao, Michael and Le, Anh and Kallus, Nathan and others},
  journal={arXiv preprint arXiv:2311.11922},
  year={2023}
}

@article{prentice1989surrogate,
  title={Surrogate endpoints in clinical trials: definition and operational criteria},
  author={Prentice, Ross L},
  journal={Statistics in medicine},
  volume={8},
  number={4},
  pages={431--440},
  year={1989},
  publisher={Wiley Online Library}
}

@article{Baker2018-FiveCriteriaUsingSurrogate,
  title = {Five Criteria for Using a Surrogate Endpoint to Predict Treatment Effect Based on Data from Multiple Previous Trials},
  journal = {Statistics in medicine},
  author = {Baker, Stuart G.},
  date = {2018-02-20},
  year = {2018},
  doi = {10.1002/sim.7561},
}

@article{papamakarios2021normalizing,
  title={Normalizing flows for probabilistic modeling and inference},
  author={Papamakarios, George and Nalisnick, Eric and Rezende, Danilo Jimenez and Mohamed, Shakir and Lakshminarayanan, Balaji},
  journal={Journal of Machine Learning Research},
  volume={22},
  number={57},
  pages={1--64},
  year={2021}
}

@article{scikitlearn,
  title={Scikit-learn: Machine learning in Python},
  author={Pedregosa, Fabian and Varoquaux, Ga{\"e}l and Gramfort, Alexandre and Michel, Vincent and Thirion, Bertrand and Grisel, Olivier and Blondel, Mathieu and Prettenhofer, Peter and Weiss, Ron and Dubourg, Vincent and others},
  journal={the Journal of machine Learning research},
  volume={12},
  pages={2825--2830},
  year={2011},
  publisher={JMLR. org}
}

@article{pytorch,
  title={Pytorch: An imperative style, high-performance deep learning library},
  author={Paszke, Adam and Gross, Sam and Massa, Francisco and Lerer, Adam and Bradbury, James and Chanan, Gregory and Killeen, Trevor and Lin, Zeming and Gimelshein, Natalia and Antiga, Luca and others},
  journal={Advances in neural information processing systems},
  volume={32},
  year={2019}
}

@inproceedings{aglin2021pydl8,
  title={Pydl8. 5: a library for learning optimal decision trees},
  author={Aglin, Ga{\"e}l and Nijssen, Siegfried and Schaus, Pierre},
  booktitle={Proceedings of the Twenty-Ninth International Conference on International Joint Conferences on Artificial Intelligence},
  pages={5222--5224},
  year={2021}
}

@book{friedman2015fundamentals,
  title={Fundamentals of clinical trials},
  author={Friedman, Lawrence M and Furberg, Curt D and DeMets, David L and Reboussin, David M and Granger, Christopher B},
  year={2015},
  publisher={Springer}
}

@article{athey2025surrogate,
  title={The surrogate index: Combining short-term proxies to estimate long-term treatment effects more rapidly and precisely},
  author={Athey, Susan and Chetty, Raj and Imbens, Guido W and Kang, Hyunseung},
  journal={Review of Economic Studies},
  pages={rdaf087},
  year={2025},
  publisher={Oxford University Press UK}
}

@article{martin2008long,
  title={Long-term maternal effects of early childhood intervention: Findings from the Infant Health and Development Program (IHDP)},
  author={Martin, Anne and Brooks-Gunn, Jeanne and Klebanov, Pamela and Buka, Stephen L and McCormick, Marie C},
  journal={Journal of Applied Developmental Psychology},
  volume={29},
  number={2},
  pages={101--117},
  year={2008},
  publisher={Elsevier}
}

@article{hill2011bayesian,
  title={Bayesian nonparametric modeling for causal inference},
  author={Hill, Jennifer L},
  journal={Journal of Computational and Graphical Statistics},
  volume={20},
  number={1},
  pages={217--240},
  year={2011},
  publisher={Taylor \& Francis}
}

@article{10.1001/jama.1990.03440220059030,
    author = {"IHDP"},
    title = {Enhancing the Outcomes of Low-Birth-Weight, Premature Infants: A Multisite, Randomized Trial},
    journal = {JAMA},
    volume = {263},
    number = {22},
    pages = {3035-3042},
    year = {1990},
    month = {06},
    issn = {0098-7484},
    doi = {10.1001/jama.1990.03440220059030},
    eprint = {https://jamanetwork.com/journals/jama/articlepdf/382131/jama_263_22_030.pdf},
}

@article{laurent1992review,
  title={Review of validity research on the Stanford-Binet Intelligence Scale},
  author={Laurent, Jeff and Swerdlik, Mark and Ryburn, Mary},
  journal={Psychological assessment},
  volume={4},
  number={1},
  pages={102},
  year={1992},
  publisher={American Psychological Association}
}

@article{coley2020dementia,
  title={Dementia risk scores as surrogate outcomes for lifestyle-based multidomain prevention trials—rationale, preliminary evidence and challenges},
  author={Coley, Nicola and Hoevenaar-Blom, Marieke P and van Dalen, Jan-Willem and Moll van Charante, Eric P and Kivipelto, Miia and Soininen, Hilkka and Andrieu, Sandrine and Richard, Edo and {PRODEMOS consortium, the preDIVA study group, the MAPT/DSA group, and the HATICE consortium}},
  journal={Alzheimer's \& Dementia},
  volume={16},
  number={12},
  pages={1674--1685},
  year={2020},
  publisher={Wiley Online Library}
}
\bibliographystyle{plainnat}
\clearpage
%
\onecolumn
\appendix

\section{Notation}
\begin{table}[H]
    \centering
    \caption{A summary of used notations throughout the paper.}
    \label{tab:notation}   
    \vspace{0.5em}
    \begin{tabular}{ll}
    \toprule
    \multicolumn{2}{l}{Random variables} \\
    \midrule
        $X$ & Pre-treatment variables  \\
        $T$ & Treatment variable $T \in \set{0,1}$  \\
        $S$, $S(t)$ & Surrogate variables and surrogate potential outcomes \\
        $S_{\med}$ & Mediator surrogates, ($T\to S \to Y$)\\
        $S_{\leaf}$ & Leaf surrogates (only $T \to S$)\\
        $S_{\proxy}$ & Proxy surrogates (only $S \to Y$)\\
        $Y$ & Observed outcome \\
        $Y(t)$ & Potential (interventional) outcomes \\
        $C$ & Learned composite endpoint \\
        $\Delta_Y, \Delta_S$ & $Y(1) - Y(0)$, $S(1) - S(0)$ \\
        $U$ & Unobserved confounders $\tau$ \\
        $\tau_Y(x)$ & Conditional average treatment effect (CATE) for $Y$ \\
        $\tau_f(x)$ & CATE estimate of a ``plug-in'' surrogate function $f$ \\
        $\tau_S(x)$ & CATE for $S$ \\
        $\mu_Y$, $\mu_{f(s)}$ & The expected potential outcomes $Y(t)$ for $Y$ and $f$ respectively \\
        $\mu_{Y}(x)$, $\mu_{f(s)}(x)$ & Conditional outcomes $\E[Y(t)|X=x]$ for $Y$ and $f$ respectively \\
    \midrule
    \multicolumn{2}{l}{Observations and constants} \\
    \midrule
        $w(x,s)$ & Weights of any weighted regression \\
        $w^2(x,s)$ & Weights of the \algbound{},   $\left(\frac{p(T=1 \mid s,x)}{p(T=1 \mid x)} - \frac{p(T=0 \mid s,x)}{p(T=0 \mid x)}\right)^2$ \\
        $w^+(x,s)$ & Weights of the \algbound{},  $\frac{p(T=1 \mid s,x)}{p(T=1 \mid x)} + \frac{p(T=0 \mid s,x)}{p(T=0 \mid x)}$ \\
        $w^1(x,s)$ & Weights of the \algbound{},  $\left\lvert \frac{p(T=1 \mid s,x)}{p(T=1 \mid x)} - \frac{p(T=0 \mid s,x)}{p(T=0 \mid x)}\right\rvert$ \\
        $\rho(s,x)$ & Surrogate score $p(T=1|x, s)$\\
        $e(x)$ & Propensity score,  $p(T=1 \mid x)$ \\
        $\eta(x,s)$ & Difference of densities $\frac{p(T=1 \mid s,x)}{p(T=1 \mid x)} - \frac{p(T=0 \mid s,x)}{p(T=0 \mid x)}$\\
        $\pi(x,s)$ & Difference of densities  $p(s\mid x, T=1) = p(s\mid x, T=0)$ \\
        $R^o_\tau(f)$, $R^e_\tau(f)$ & Expected risk of the CATE estimate for $f$ in the observational and trial data \\
        $p^o(x)$, $p^e(x)$,  & Marginal distribution of the covariates in the observational data and trial respectively \\
        $\tau^o$, $\tau^e$, $\mu^o$, $\mu^e$ & Observational and trial counterparts of treatment effects and expected outcomes \\
        $\E^o$, $\E^e$ & Observational and trial counterparts of expectations \\
        $D_t^o$, $D^o_y$ & Observational treatment and outcome data respectively\\
        $m_t$, $m_y$ & Number of samples for treatment and outcome data respectively\\
        $x_i$ & Observed pretreatment variables \\
        $t_i$ & Observed treatments \\
        $s_i$ & Observed surrogates \\
        $y_i$ & Observed outcomes \\
        $k$ & Number of pretreatment variables \\
        $d$ & Number of surrogate variables \\
        $L$ & Number of drawn samples in the \algsamp{} methods   \\
    \midrule
    \multicolumn{2}{l}{Functions} \\
    \midrule
    $f(s)$ & ``Plug-in'' surrogate functions learned via \algsamp{}  and \algbound{} methods \\ 
    $h(x,s)$ & Surrogate index $h:= \E[Y|X=x, S=s]$ \\
    \bottomrule
    \end{tabular}    
\end{table}

\newpage

\section{Proofs}
\label{app:proofs}

\subsection{On the identifiability of surrogate effects}
\label{app:identifiability}
Definitions of surrogacy that depend on knowing the joint distribution of surrogate and primary-outcome effects, $\Delta_S = S(1)-S(0)$ and $\Delta_Y=Y(1)-Y(0)$, will not in general be identifiable from observational or experimental data. It is well-known that even the joint distribution $p(Y(0), Y(1))$ is not identifiable other than under strong assumptions, and neither $p(Y(1)-Y(0))$. Thus, neither is $p(Y(1)-Y(0), S(1)-S(0))$. 

If there exists a variable $X$ that captures non-treatment-related influence on $Y$ such that it renders $Y(1), Y(0)$ conditionally independent, i.e., $Y(0) \indep Y(1) \mid X$, identification of $p(Y(0),Y(1))$ is possible, and $p(Y(0),Y(1)) = p(Y(1) \mid X)p(Y(0) \mid X)$. However, for many causal models, such an $X$ cannot exist, such as in structural causal models with additive (unobservable) noise, since this noise inextricably links $Y(0)$ and $Y(1)$.

It is, however possible to identify $p(Y(t) \mid X)$ and $p(S(t) \mid X)$ for some moderator $X$ which can be used to define an approximation of the joint distribution of effects. First, let
\begin{align*}
\tilde{p}(y_1, y_0 \mid x) & = p(Y(1)=y_1 \mid x)p(Y(0)=y_0 \mid x) \\
\tilde{p}(s_1, s_0 \mid x) & = p(S(1)=s_1 \mid x)p(S(0)=s_0 \mid x)
\end{align*}
and then define a pseudo-joint distribution 
\begin{align*}
\tilde{p}_{X}(\Delta_Y=\delta_y, \Delta_S=\delta_s) 
\coloneqq \sum_{\substack{x,y_0, y_1, s_0, s_1 : \\ y_1-y_0=\delta_y\\s_1-s_0=\delta_s}} p(x) \tilde{p}(y_1, y_0 \mid x)\tilde{p}(s_1, s_0 \mid x).
\end{align*}
If $X$ (however unlikely) renders all potential outcomes conditionally independent, this quantity coincides with the true joint distribution of surrogate and main effects. We do not pursue identification of $\tilde{p}_{X}(\Delta_Y, \Delta_S)$, however.

\subsection{Proof of Proposition~\ref{prop:cate}}
\begin{repprop}{prop:cate}
    Under Assumptions~\ref{asmp:id}--\ref{asmp:comp},
    \begin{align*}
    \tau_{Y}^e(x) = \tau_Y^o(x) & = \E_S^o[h^o(x,S) \mid X=x, T=1] \\
    & - \E_S^o[h^o(x,S) \mid X=x, T=0] \\
    \tau_{f(S)}^e(x) = \tau_{f(S)}^o(x) & = \E_S^o[f(S) \mid X=x, T=1] \\
    & - \E_S^o[f(S) \mid X=x, T=0]~.
    \end{align*}
\end{repprop}

\begin{proof}
We begin with $\tau_Y^o(x)$ given by 
    \begin{align}
    \tau_Y^o(x) &= \E[\Delta_Y| P=o] =\E^o[Y(1) - Y(0)] \\ &\underset{(a)}{=} \E^o[Y \mid X=x, T=1] - \E^o[Y \mid X=x, T=0] \\ 
    &= \E^o_S[\E^o[Y \mid S, X=x, T=1] - \E_S^o [\E^o[Y \mid S, X=x, T=0]] \\
    &\underset{(b)}{=} \E^o_S[\E^o[Y \mid S, X=x] \mid X=x, T=1] - \E_S^o [\E^o[Y \mid S, X=x] \mid X=x, T=0] \label{_eq:prop1-proof-ref} \\
    &\underset{(c)}{=} \E_S^o[h^o(x,S) \mid X=x, T=1]- \E_S^o[h^o(x,S) \mid X=x, T=0],
    \end{align}
    where we use~\Cref{asmp:id} in (a) and~\Cref{asmp:med} in (b), and (c) follows from the definition of $h^o$. Combining \eqref{_eq:prop1-proof-ref} with~\Cref{asmp:comp} yields $\tau_Y^e(x) =\tau_Y^o(x)$. Similarly for $\tau^o_{f(S)}(x)$ we have
    \begin{align*}
    \tau^o_{f(S)}(x) &= \E_S^o[f(S) \mid X=x, T=1]- \E_S^o[f(S) \mid X=x, T=0] \\ 
        &= \E_S^e[f(S) \mid X=x, T=1]- \E_S^e[f(S) \mid X=x, T=0] = \tau^e_{f(S)}(x).
    \end{align*}
\end{proof}

\subsection{The trial CATE risk in the linear case}
\label{app:linear_risk}
When the functions $f$ and $h$ are linear with respect to $s$, i.e. $h(x, s) = h_x(x) + \beta_h^\top, f(s) =\beta^\top s$, the trial CATE risk simplifies to a linear function of the surrogate CATE:
\begin{align*}
R_\tau^e(f) &= \E^e_X[(\E_S[h(X, S) -f(S) \mid X, T=1] \\
& \qquad\; - \E_S[h(X, S) -f(S) \mid X, T=0] )^2] \\
&= \E^e_X[(\E_S[h_x(X) + \beta_h^\top S - \beta_f^\top S \mid X, T=1] \\
& \qquad\; - \E_S[h_x(X) + \beta_h^\top S - \beta_f^\top S \mid X, T=0] )^2] \\
&= \E^e_X[((\beta_h-\beta_f)^\top\E_S[ S \mid X, T=1] \\
& \qquad\; - (\beta_h-\beta_f)^\top \E_S[S \mid X, T=0] )^2]\\
&= \E^e_X[((\beta_h-\beta_f)^\top \tau_S(x))^2].
\end{align*}
\qed

\subsection{Proof of Proposition~\ref{prop:risk_bound}}
\label{app:risk_bound}

\begin{repprop}{prop:risk_bound} 
    For any surrogate function $f$, it holds under Assumptions~\ref{asmp:id}--\ref{asmp:comp} that
    \begin{align}
    R^o_\tau(f) &\leq \E_{X,S}[w^2(X,S)(h(X,S) - f(S))^2] \label{eq:app_cate_bound} \\
    & \leq \E_{X,S}[w^2(X,S)(Y - f(S))^2] + \sigma^2 \nonumber
    \end{align}
    where $w^2(x,s) = \left(\frac{p(T=1 \mid s,x)}{p(T=1 \mid x)} - \frac{p(T=0 \mid s,x)}{p(T=0 \mid x)}\right)^2$, and the second inequality holds if the outcome has exogenous, centered, additive noise: $Y = h(X,S)+\epsilon$ where $\E[\epsilon]=0$, $\mbox{Var}(\epsilon) \leq \sigma^2$ and $\epsilon \indep S, X, T$. If $p^o(X)/p^e(X)$ is known, the same expression with the weight $\tilde{w}^2(x,s) = \frac{p^o(x)}{p^e(x)}w^2(x,s)$ yields a bound on $R^e_\tau(f)$.
\end{repprop}

\begin{proof}
    We leave out superscripts indicating the population, as the proof holds for either cohort as long as it is used consistently on both the LHS and RHS. 

    Under Assumptions~\ref{asmp:id}--\ref{asmp:med}, we have consistency $S = S(T)$ and exchangeability on $S$, $S(t) \indep T \mid X$, as well as $Y = Y(t)$, and $Y(t) \indep T \mid X$, and $Y \indep T \mid X, S$. Then, we can prove
    \begin{align}
    \E[Y(t) \mid X=x] &= \E[Y(t) \mid X=x, T=t] \\
    & = \E[Y \mid X=x, T=t] \\
    & = \E_{S(t)}[\E[Y \mid X=x, S(t), T=t] \mid X=x, T=t] \\
    & = \E_{S}[\E[Y \mid X=x, S, T=t] \mid X=x, T=t] \label{eq:surr_dens}\\
    & = \E_{S}[\E[Y \mid X=x, S] \mid X=x, T=t]~.
    \end{align}
    Where we have used in \eqref{eq:surr_dens} that 
    $$
    p(S(t)=s \mid X=x) = p(S=s \mid X=x, T=t)~.
    $$
    under the given assumptions.    
    Let $h(x, s)  \coloneqq \E[Y \mid X=x, S=s]$. Then,     
    \begin{align*}
    (\tau_Y(x) - \tau_{f(S)}(x))^2
    & = (\E[h(x, S(1))-f(S(1)) \mid X=x] 
     - \E[h(x, S(0))-f(S(0)) \mid X=x])^2 \\
    &  = \left(\E_{S}[(h(x, S) -f(S)) \mid X=x, T=1] 
     - \E_{S}[(h(x, S) -f(S)) \mid X=x, T=0] \right)^2 \\
    & = \left(\E[\frac{p(S\mid T=1, x)}{p(S\mid x)}(h(x, S) -f(S)) \mid X=x] 
     - \E[\frac{p(S\mid T=0, x)}{p(S\mid x)}(h(x, S) -f(S)) \mid X=x]\right)^2 \\
    & = \left(\E\left[\left(\frac{p(T=1 \mid S, x)}{p(T=1\mid x)} - \frac{p(T=0 \mid S, x)}{p(T=0\mid x)}\right) (h(x, S) -f(S)) \mid X=x \right] \right)^2 \\
    & \leq \E[w^2(x,s)(h(x, S) -f(S))^2 \mid X=x]
    \end{align*}
    where $w^2(x,s) = \left( \frac{p(T=1 \mid S, x)}{p(T=1\mid x)} - \frac{p(T=0 \mid S, x)}{p(T=0\mid x)} \right)^2$, and so 
    $$
    \E_X[(\tau_Y(x) - \tau_{f(S)}(X))^2] \leq \E_{X,S}[w^2(X,s)(h(X, S) -f(S))^2]~.
    $$
    To arrive at the result using $Y$ instead of $h(X, S) $, we exploit the assumption of additive, centered, exogenous noise $\epsilon$, and that $Y = h(X, S)  + \epsilon$ in this case. For any weighting function, 
    \begin{align*}
    \E_{X,S}[w(X,S)(h(X, S) -f(S))^2] &= \E_{X,S,Y}[w(X,S)(Y-f(S) + h(X, S)  - Y)^2]  \\
    & = \E_{X,S,Y}[w(X,S)[(Y-f(S))^2 + 2(h(X, S)  - Y) + (h(X, S) -Y)^2]] \\
    & \leq \E_{X,S,Y}[w(X,S)[(Y-f(S))^2]] + 0 + \sigma^2
    \end{align*}
    by the assumption of bounded noise from the statement. 
\end{proof}

We can prove a similar result by bounding the square earlier, 
\begin{align*}
\frac{1}{2}(\tau_Y(x) - \tau_{f(S)}(x))^2
& = \frac{1}{2}(\E[h(x, S(1))-f(S(1)) \mid X=x] 
 - \E[h(x, S(0))-f(S(0)) \mid X=x])^2 \\
& \leq \E[h(x, S(1))-f(S(1)) \mid X=x]^2 
 + \E[h(x, S(0))-f(S(0)) \mid X=x]^2 \\
& \leq \E_{S(1)}[(h(x, S(1))-f(S(1)))^2 \mid X=x] 
 +\E_{S(1)}[(h(x, S(0))-f(S(0)))^2 \mid X=x] \\
& \leq \E_{S}[(h(x, S) -f(S))^2 \mid X=x, T=1] 
 +\E_{S}[(h(x, S) -f(S))^2 \mid X=x, T=0] \\
& = \E\left[\frac{p(S\mid T=1, x)}{p(S\mid x)}(h(x, S) -f(S))^2 \mid X=x\right] \\ 
 &\quad +\E\left[\frac{p(S\mid T=0, x)}{p(S\mid x)}(h(x, S) -f(S))^2 \mid X=x\right] \\
& = \E\left[\left(\frac{p(T=1 \mid S, x)}{p(T=1\mid x)} + \frac{p(T=0 \mid S, x)}{p(T=0\mid x)}\right) (h(x, S) -f(S))^2 \mid X=x \right] \\
& = \E\left[w^+(x,s)(h(x, S) -f(S))^2 \mid X=x\right]
\end{align*}
where $w^+(x,s) = \frac{p(T=1 \mid S, x)}{p(T=1\mid x)} + \frac{p(T=0 \mid S, x)}{p(T=0\mid x)}$. 
It follows immediately that 
$$
\E_X[(\tau_Y(x) - \tau_{f(S)}(x))^2] \leq 2\E_{X,S}[w^+(x,s)(h(x, S) -f(S))^2]~.
$$
However, the weights $w^+$ are all uniform if $e(x) = p(T=1 \mid x) = 0.5$, that is, if there is no confounding.

Using a very similar strategy, we can also prove a result for the L1-loss. 
\begin{align*}
    |\tau_Y(x) - \tau_{f(S)}(x)| &= \left| \int_s \left(\frac{p(T=1 \mid S, x)}{p(T=1\mid x)} - \frac{p(T=0 \mid S, x)}{p(T=0\mid x)}\right)p(s\mid x)(h(x, s)  - f(s))ds \right|\\
    & = |\E_S[\eta(x,S)(h(x, S)  - f(S)) \mid X=x]| \\
    & \leq \E_S[w^1(x,S)|h(x, S)  - f(S)| \mid X=x]
\end{align*}
with $\eta(x,s) = \frac{p(T=1 \mid S, x)}{p(T=1\mid x)} - \frac{p(T=0 \mid S, x)}{p(T=0\mid x)}$, and $w^1(x,s) = |\eta(x,s)|$, 
which means
$$
\E_X[|\tau_Y(X) - \tau_{f(S)}(X)|] \leq \E_{X,S}[w^1(x,s)|h(x, S)  - f(S)|]~.
$$
\newpage
\section{Examples and counterexamples}
\label{app:examples}
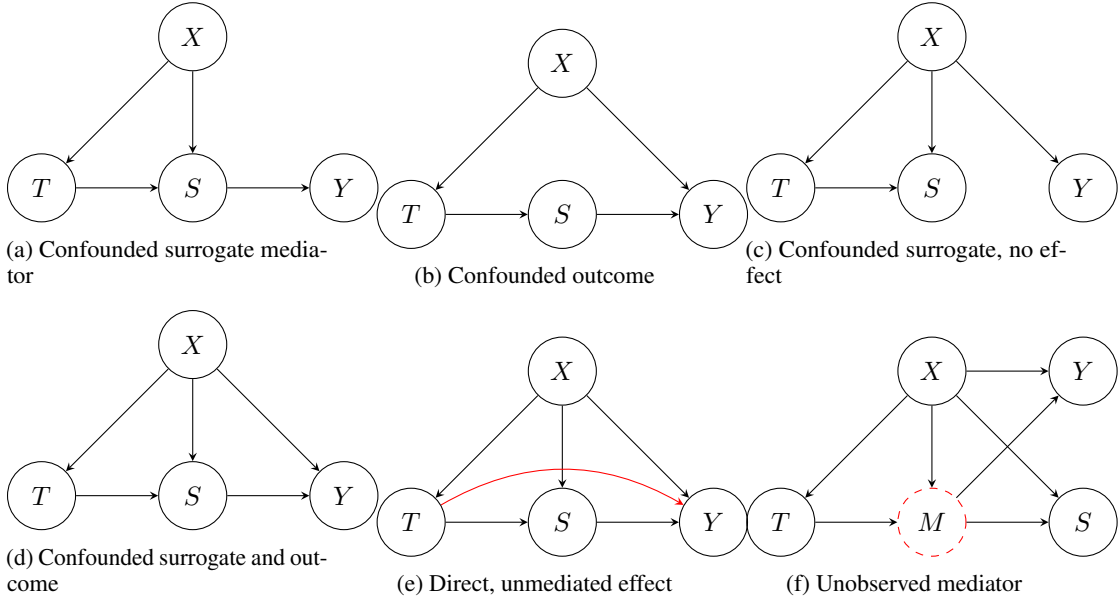
\begin{figure}[htbp]
\centering

\begin{subfigure}{0.3\textwidth}
\centering
\begin{tikzpicture}[
    ->,
    >=stealth,
    node distance=2cm,
    every node/.style={circle, draw, minimum size=0.9cm}
]
\node (T) {$T$};
\node (S) [right of=T] {$S$};
\node (X) [above of=S] {$X$};
\node (Y) [right of=S] {$Y$};
\draw (X) -- (T);
\draw (T) -- (S);
\draw (X) -- (S);
\draw (S) -- (Y);
\end{tikzpicture}
\caption{Confounded surrogate mediator}
\end{subfigure}
\hfill
\begin{subfigure}{0.3\textwidth}
\centering
\begin{tikzpicture}[
    ->,
    >=stealth,
    node distance=2cm,
    every node/.style={circle, draw, minimum size=0.9cm}
]
\node (T) {$T$};
\node (S) [right of=T] {$S$};
\node (X) [above of=S] {$X$};
\node (Y) [right of=S] {$Y$};
\draw (X) -- (T);
\draw (T) -- (S);
\draw (X) -- (Y);
\draw (S) -- (Y);
\end{tikzpicture}
\caption{Confounded outcome}
\end{subfigure}
\hfill
\begin{subfigure}{0.3\textwidth}
\centering
\begin{tikzpicture}[
    ->,
    >=stealth,
    node distance=2cm,
    every node/.style={circle, draw, minimum size=0.9cm}
]
\node (T) {$T$};
\node (S) [right of=T] {$S$};
\node (X) [above of=S] {$X$};
\node (Y) [right of=S] {$Y$};
\draw (X) -- (T);
\draw (T) -- (S);
\draw (X) -- (S);
\draw (X) -- (Y);
\end{tikzpicture}
\caption{Confounded surrogate, no effect}
\end{subfigure}

\medskip

\begin{subfigure}{0.3\textwidth}
\centering
\begin{tikzpicture}[
    ->,
    >=stealth,
    node distance=2cm,
    every node/.style={circle, draw, minimum size=0.9cm}
]
\node (T) {$T$};
\node (S) [right of=T] {$S$};
\node (X) [above of=S] {$X$};
\node (Y) [right of=S] {$Y$};
\draw (X) -- (T);
\draw (T) -- (S);
\draw (X) -- (S);
\draw (X) -- (Y);
\draw (S) -- (Y);
\end{tikzpicture}
\caption{Confounded surrogate and outcome}
\end{subfigure}
\hfill
\begin{subfigure}{0.3\textwidth}
\centering
\begin{tikzpicture}[
    ->,
    >=stealth,
    node distance=2cm,
    every node/.style={circle, draw, minimum size=0.9cm}
]
\node (T) {$T$};
\node (S) [right of=T] {$S$};
\node (X) [above of=S] {$X$};
\node (Y) [right of=S] {$Y$};
\draw (X) -- (T);
\draw (T) -- (S);
\draw (X) -- (S);
\draw (X) -- (Y);
\draw (S) -- (Y);
\draw [red, bend left] (T) to (Y); 
\end{tikzpicture}
\caption{Direct, unmediated effect}
\end{subfigure}
\hfill
\begin{subfigure}{0.3\textwidth}
\centering
\begin{tikzpicture}[
    ->,
    >=stealth,
    node distance=2cm,
    every node/.style={circle, draw, minimum size=0.9cm}
]
\node (T) {$T$};
\node (M) [right of=T, dashed, draw=red] {$M$};
\node (S) [right of=M] {$S$};
\node (X) [above of=M] {$X$};
\node (Y) [right of=X] {$Y$};
\draw (X) -- (T);
\draw (T) -- (M);
\draw (X) -- (M);
\draw (X) -- (Y);
\draw (X) -- (S);
\draw (M) -- (Y);
\draw (M) -- (S);
\end{tikzpicture}
\caption{Unobserved mediator}
\end{subfigure}

\caption{Six scenarios for surrogate learning.}
\label{fig:six-dags}
\end{figure}

Figure~\ref{fig:six-dags} illustrates the causal graphs of six surrogate learning scenarios. We go through the properties of cases a-e) below. In all of the cases illustrated, for any function $f(S)$
\begin{align*}
\mu_{f(S)}(x) & = \E[f(S(t)) \mid X=x] = \E[f(S) \mid X=x, T=t] \\
\tau_{f(S)}(x) & = \E[f(S) \mid X=x, T=1] - \E[f(S) \mid X=x, T=0] \\
\end{align*}
We also note that a convenient form to express $R^e_\tau(f)$ is
$$
R_\tau^e(f) = \E^o_X\bigg[ \frac{p_e(X)}{p_o(X)} \big(\int_s \pi(x,s)(h(X, S) -f(S))ds\big)^2 \bigg]~, 
$$
where 
$$
\pi(x,s) \coloneqq p(s \mid x, T=1) - p(s \mid x, T=0)~.
$$

\subsection{a) Counfounded surrogate mediator}
In this case, 
$$
\mu_Y(x) = \E[Y(t) \mid X=x] = \E_{S(t)}[\E[Y \mid S(t)] \mid X=x] = \E_S[\E[Y \mid S] \mid X=x, T=t]
$$
and with $g(s) = \E[Y \mid S=s]$, the CATE on $Y$ is equal to 
$$
\tau_Y(x) = \E_S[g(S) \mid T=1, x] - \E_S[g(S) \mid T=0, x]
$$
which is non-zero in general. Moreover, there exists an $f(S)$ such that $\tau_Y(x) = \tau_{f(S)}(x)$ for all $x$, and that $f=g + c$ for any constant $c$. Thus, $f=g+c$ are also the minimizers of \eqref{eq:main_obj}. $f=g$ is also the minimizer of any weighted regression objective $\E[w(X,S)(h(X, S)  - f(S))^2]$ when $h(X, S)  = g(S)$.

\subsection{b) Counfounded outcome}\label{app:examples-b}

In this case, 
$$
\mu_Y(x) = \E[Y(t) \mid X=x] = \E_{S(t)}[\E[Y \mid S(t), X=x, T=t] \mid X=x] = \E_S[\E[Y \mid S, X=x] \mid X=x, T=t]
$$
and with $h(x, s)  = \E[Y \mid S=s, X=x]$,
$$
\tau_Y(x) = \E_S[h(x, S)  \mid T=1, x] - \E_S[h(x, S)  \mid T=0, x]
$$
which is non-zero in general. In this case, there {\color{red}\emph{may not exist}} a function $f(S)$ such that $\tau_Y(x) = \tau_{f(S)}(x)$ for all $x$, since $X$ carries information about $Y$ not mediated through $S$. In fact, $\E[f(S) \mid X=x, T=t] = \E[f(S) \mid T=t]$ is constant with respect to $X$ since $S \indep X \mid T$. Under mild assumptions, a minimizer of \eqref{eq:main_obj} yields the ATE,
$$
\forall f \in f^* : \tau_{f(S)}(x) = \E_X[\E[h(X, S)  \mid T=1, X] - \E[h(X, S)  \mid T=0, X]] = \tau_Y~.
$$

We can find the minimizer $f^*$ by setting the derivative of any weighted regression to zero, conditioned on $s$,
$$
\E_X[2w(X,s)(h(X, s) -f(s)) \mid S=s] = 0 \Longrightarrow f^*(s) = \frac{\E_X[w(X,s)h(X, s)  \mid S=s]}{\E_X[w(X,s) \mid S=s]}.
$$
This expression is not constant with respect to $s$ in general. Then
\begin{align*}
\tau_{f^*(S)}(x) & = \E[f^*(S) \mid T=1,x] - \E[f^*(S) \mid T=0,x] \\
& = \E_S\left[\frac{\E_{X'}[w(X',S)h(X,S') \mid S]}{\E_{X'}[w(X',S) \mid S]} \mid T=1, \not x \right] - \E_S\left[\frac{\E_{X'}[w(X',S)h(X,S') \mid S]}{\E_{X'}[w(X',S) \mid S]} \mid T=0, \not x\right] 
\end{align*}
The conditioning on $X$ in the outer expectations can be removed since $S \indep X \mid T$ in that setting. Therefore, the learned $\tau_{f(S)}(x)$ is indeed constant w.r.t. $x$ for setting $b)$, but not necessarily equal to the true ATE. 

\subsubsection*{Weights that yield the true ATE}
Are there weights $w$ that result in an unbiased estimate of the true ATE?

\begin{align*}
\tau_{f^*(S)} & = \E[\tau_{f^*(S)}(X)] \\
&= \int_{x' }p(x')\int_s \frac{\int_x w(x,s)h(x, s)  p(x\mid s)dx}{\int_x w(x,s)p(x\mid s)dx} \left( p(s\mid T=1) - p(s\mid T=0)\right) ds dx' \\
&= \int_s \frac{\int_x w(x,s)h(x, s)  p(x\mid s)dx}{\int_x w(x,s)p(x\mid s)dx} \left( p(s\mid T=1) - p(s\mid T=0)\right) ds \\
& = \{ w(x,s) = p(x)/p(x\mid s) \}  \\
&= \int_s \frac{\int_x \frac{p(x)}{p(x\mid s)}h(x, s)  p(x\mid s)dx}{\int_x \frac{p(x)}{p(x\mid s)}p(x\mid s)dx} \left( p(s\mid T=1) - p(s\mid T=0)\right) ds \\
&= \int_s \int_x p(x)h(x, s)  dx \left( p(s\mid T=1) - p(s\mid T=0)\right) ds \\
&= \E_X[\E_S[h(X,S) \mid T=1]]-\E_X[\E_S[h(X,S) \mid T=0] \\
&= \tau_Y.
\end{align*}

Thus, for the choice of $w(x,s) = \frac{p(x)}{p(x|s)}$ the weighted regression will be an unbiased estimator of the ATE. But this is not the only choice.

Consider the weights $w^2(x,s) = (\frac{p(T=1 \mid s,x)}{p(T=1 \mid x)} - \frac{p(T=0 \mid s,x)}{p(T=0 \mid x)})^2$ and $w^+(x,s) = \frac{p(T=1 \mid s,x)}{p(T=1 \mid x)} + \frac{p(T=0 \mid s,x)}{p(T=0 \mid x)}$ for the CATE upper bound, described after \Cref{prop:risk_bound}. We ask what happens to the ATE estimate for the choice of $w^+(x,s)$. We begin noting that:

$$
w^2(x,s) \underset{(a)}{=} \left(\frac{p(s|T=1,x)}{p(s|x)} - \frac{p(s|T=0,x)}{p(s|x)} \right)^2 \underset{(b)}{=} \frac{(p(s|T=0) - p(s|T=1))^2}{p(s|x)^2},
$$
$$
w^+(x,s) \underset{(a)}{=} \frac{p(s|T=1,x)}{p(s|x)} + \frac{p(s|T=0,x)}{p(s|x)} \underset{(b)}{=} \frac{p(s|T=0) + p(s|T=1)}{p(s|x)},
$$
where (a) and (b) follows from Bayes' rule and $X \indep S \mid T$ respectively.
Looking at the ATE for $f$ with $w^+$: 
\begin{align*}
\tau_{f(S)} &= \int_s \frac{\int_x w^+(x,s)h(x, s)  p(x\mid s)dx}{\int_x w^+(x,s)p(x\mid s)dx} \left( p(s\mid T=1) - p(s\mid T=0)\right) ds \\
&\underset{(c)}{=} \int_s \frac{\int_x \frac{p(s|T=0) + p(s|T=1)}{p(s|x)}  h(x, s)  p(x\mid s)dx}{\int_x \frac{p(s|T=0) + p(s|T=1)}{p(s|x)} p(x\mid s)dx} \left( p(s\mid T=1) - p(s\mid T=0)\right) ds \\
&\underset{(d)}{=} \int_s \frac{\int_x \frac{p(x)}{p(s)}  h(x, s)  dx}{\int_x \frac{p(x)}{p(s)} dx} \left( p(s\mid T=1) - p(s\mid T=0)\right) ds \\
&= \int_s \int_x p(x)h(x, s)  dx \left( p(s\mid T=1) - p(s\mid T=0)\right) ds \\
&= \E_X[\E_S[h(X,S) \mid T=1]]-\E_X[\E_S[h(X,S) \mid T=0] \\
&= \tau_Y,
\end{align*}
where the terms that depend only on $s$ cancel in $(c)$ and $(d)$. Thus, in this case, the sum weights, $w^+$ are an unbiased estimator of the ATE. The same is true for any weighting function that can be written as $w(x,s) = \phi(s,t)/p(s \mid x)$, such as the absolute difference weights $w^1(x,s) = |\frac{p(T=1 \mid s,x)}{p(T=1 \mid x)} - \frac{p(T=0 \mid s,x)}{p(T=0 \mid x)}|$

Following a similar derivation for $w^2$, we get 
\begin{align*}
\tau_{f(S)} &= \int_s \frac{\int_x \frac{p(x)}{p(s\mid x)}h(x, s) dx}{\int_x \frac{p(x)}{p(s\mid x)}dx} \left( p(s\mid T=1) - p(s\mid T=0)\right) ds 
\end{align*}
which is not equal to $\tau_Y$ in general. However, in a randomized experiment version of case b), $S \indep X$ and $p(s \mid x) = p(s)$. In this case, $\tau_{f(S)} = \tau_Y$. We leave the task of characterizing the precise bias in the general case for future work. 

\subsection{c) Counfounded surrogate, no effect}
In this case, 
$$
\mu_Y(x) = \E[Y(t) \mid X=x] = \E_{S(t)}[\E[Y \mid S(t), X=x] \mid X=x] = \E_S[\E[Y \mid X=x] \mid X=x, T=t] = \E[Y \mid X=x] 
$$
since $Y \indep T \mid X$ and $Y \indep S \mid X$. Since the potential outcomes are independent of $T$, with $h(x) = \E[Y \mid X=x]$
$$
\tau_Y(x) = h(x) - h(x) = 0.
$$
Trivially, there exists an $f(S)$ such that $\tau_Y(x) = \tau_{f(S)}(x)$ for all $x$, and that is any constant $f(S)=c$. Thus, $f=g$ is a minimizer of \eqref{eq:main_obj}.

Since $Y \indep S \mid X$, $h(x, s)  = h(x)$. We can find the minimizer $f^*$ of \eqref{eq:main_obj} by setting the derivative of any weighted regression to zero, conditioned on $s$,
$$
\E[2w(X,s)(h(X)-f(s)) \mid S=s] = 0 \Longrightarrow f^*(s) = \frac{\E_X[w(X,s)h(X) \mid S=s]}{\E_X[w(X,s) \mid S=s]}.
$$
This expression is not constant with respect to $s$ in general. Then
\begin{align*}
\tau_{f^*(S)}(x) & = \E[f^*(S) \mid T=1,x] - \E[f^*(S) \mid T=0,x] \\
& = \E_S\left[\frac{\E_{X'}[w(X',s)h(X') \mid S]}{\E_{X'}[w(X',s) \mid S]} \mid T=1,X=x \right] - \E_S\left[\frac{\E_{X'}[w(X',s)h(X') \mid S]}{\E_{X'}[w(X',s) \mid S]} \mid T=0,X=x\right] \\
& \neq 0
\end{align*}
which is biased in general. 

\subsection{d) Confounded surrogate and outcome}
\label{app:examples_d}
In this case, the potential outcomes for the outcome $Y(1), Y(0)$ and surrogate $S(1), S(0)$ are independent given $X$,
$$
\mu_Y(x) = \E[Y(t)| X=x] = \E_{S(t)}[\E[Y|S(t), X=x]|X=x] \underset{(a)}{=} \E_S[\E[Y|S, x] | X=x, T=t]
$$
where in (a) we use  $S(1),S(0) \indep T | X$  and $Y \indep T|X,S$. The treatment effect for $h(x,s) = \E[Y|X=x, S=s]$ is
$$
\tau_Y(x) = \E_S[h(x, S)  \mid T=1, X=x] - \E_S[h(x, S)  \mid T=0, X=x]
$$
which is non-zero in general, but constant when $h$ is linear with respect to $x$, in that case $\tau_Y(x)$ is a constant. Since $X$ carries more information on $Y$ directly, than through $S$ it is not possible to learn an $f$ that minimizes \eqref{eq:main_obj}. We can find the minimizer $f^*$ by setting the derivative of any weighted regression to zero, conditioned on $s$,

$$
\E_X[2w(X,s)(h(X, s) -f(s)) \mid S=s] = 0 \Longrightarrow f^*(s) = \frac{\E_X[w(X,s)h(X, s)  \mid S=s]}{\E_X[w(X,s) \mid S=s]}.
$$
This expression is not constant with respect to $s$ in general. Then
\begin{align*}
\tau_{f^*(S)}(x) &= \E[f^*(S) \mid T=1,x] - \E[f^*(S) \mid T=0,x] \\
&= \E_S\left[\frac{\E_{X'}[w(X',S)h(X,S') \mid S]}{\E_{X'}[w(X',S) \mid S]} \mid T=1, X=x \right] - \E_S\left[\frac{\E_{X'}[w(X',S)h(X,S') \mid S]}{\E_{X'}[w(X',S) \mid S]} \mid T=0, X=x\right] \\ &\neq 0. 
\end{align*}

\paragraph{Unbiased weights} Before we check if a choice of $w$ results in true ATE,  we note that the following identity follow from Bayes' rule
\begin{equation}
\label{eq:bayes}
w^-(x,s) := \frac{p(T=1|x, s)}{p(T=1|x)} - \frac{p(T=0|x, s)}{p(T=0|x)}  = \frac{\pi(x,s)}{p(s|x)},
\end{equation}

where $\pi(x,s) = p(s\mid x, T=1) - p(s\mid x, T=0)$, is the difference of densities. Now, we check for the weight $w$ that result in the ATE:

\begin{align*}
\tau_{Y} - \tau_{f^*(S)} &= \E_X[\E_S[\tau_{f^*(S)}|X=x]] - \E_X[\tau_Y(x)] \\ 
&=  \E_X \left[\E_S[\tau_{f^*(S)}|X=x] - \left(\E_S[h(x, S)  \mid T=1, X=x] - \E_S[h(x, S)  \mid T=0, X=x]\right) \right] \\
&= \int_x p(x) \int_s \left(f(s) - h(x, s)  \right) \pi(x,s) ds dx \\
&= \int_s \int_x \left(f(s) - h(x, s)  \right) \underbrace{p(x)\pi(x,s)}_{Q(x,s)} dx ds \\
&= \int_s f(s) \int_x Q(x,s) dx ds - \int_s \int_x h(x,s) Q(x,s) dx ds \\
&= \int_s \frac{\int_{x'} w(x',s)h(x',s) p(x'\mid s)dx'}{\int_{x'} w(x',s)p(x'\mid s)dx'} \int_x Q(x,s) dx ds - \int_s \int_x h(x,s) Q(x,s) dx ds \\
& \mbox{Take } w = \frac{p(T=1|x, s)}{p(T=1|x)} - \frac{p(T=0|x, s)}{p(T=0|x)} \mbox{ and using \eqref{eq:bayes}:} \\
&=  \int_s \frac{\int_{x'} \frac{Q(x',s)}{p(s)p(x'|s)}h(x',s) p(x'\mid s)dx'}{\int_{x'} \frac{Q(x',s)}{p(s)p(x'|s)} p(x'\mid s)dx'} \int_x Q(x,s) dx ds - \int_s \int_x h(x,s) Q(x,s) dx ds \\
&=  \int_s \frac{\int_{x'} Q(x',s)h(x',s) dx'}{\int_{x'} Q(x',s) dx'} \int_x Q(x,s) dx ds - \int_s \int_x h(x,s) Q(x,s) dx ds  \\
&= \int_s \int_{x'} Q(x',s)h(x',s) dx' - \int_s \int_{x} Q(x,s)h(x, s)  dx \\
&= 0.
\end{align*}
However, when $\frac{p(T=1\mid x,s)}{p(t=1 \mid x)} < 1.$ the weights $w^-({x,s})$ can be negative, leading to arbitrarily poor weighted regressions, and biased CATE in general.

\subsection{e) Direct, unmediated effect}
In the case of a direct effect of $T$ on $Y$, not mediated by $S$, 
$$
\mu_Y(x) = \E[Y(t) \mid X=x] = \E_{S(t)}[\E[Y \mid S(t), X=x, T=t] \mid X=x] = \E_S[\E[Y \mid S, X=x, T=t] \mid X=x, T=t]
$$
and with a $t$ dependent $h_t(x,s) = \E[Y \mid X=x, S=s,  T=t]$,
$$
\tau_Y(x) = \E_S[h_1(x,S) \mid T=1, x] - \E_S[h_0(x,S) \mid T=0, x]
$$

There may not exist $f(S)$ such that $\tau_Y(x) = \tau_{f(S)}(x)$ for all $x$ since $X$ carries more information about $Y$. Since $h$ depends on $t$, we can find the minimizer of $f_t^*$ by setting the derivative of any weighted regression to zero, conditioned on $s$ and $t$, which yields 

$$
f_t^*(s) = \frac{\E_X[w(X,s,t)h_t(X) \mid S=s, T=t]}{\E_X[w(X,s,t) \mid S=s, T=t]}
$$
and gives the CATE estimate

\begin{align*}
\tau_{f^*(S)}(x) &= \E_S\left[\frac{\E_{X'}[w(X',S,t)h_1(S,X') \mid T=1,S]}{\E_{X'}[w(X',S,t) \mid T=1, S]} \mid T=1,  x \right] \\
&- \E_S\left[\frac{\E_{X'}[w(X',S,t)h_0(S,X') \mid T=0, S]}{\E_{X'}[w(X',S,t) \mid T=0, S]} \mid T=0, x\right] 
\end{align*}

We can check what $w$ results in true ATE,

\begin{align*}
\tau_{f^*(S)} & = \E_X[\E_S[\tau_{f^*(S)}|X=x]] = \int_x p(x) \int_s f(s) \pi(x,s) ds dx  \\
&= \int_s f(s) \int_x p(x) \pi(x,s) dx ds  \\
&= \int_s \frac{\int_{x'} w(x', s, t)h_t(s,x') p(x'\mid s)dx'}{\int_{x'} w(x',s, t)p(x'\mid s)dx'} \int_x p(x) \pi(x,s) dx ds  \\
& \mbox{Take } w(x,s,t) =  \frac{\pi(x,s)}{p(s|x)}, \mbox{ and simplify} \\
&\implies \int_s \frac{\int_{x'} h_t(s,x') p(x') \pi(x',s) dx'}{\int_{x'} p(x') \pi(x', s) dx'} \int_x p(x) \pi(x,s) dx ds  \\
&= \int_{x'} \int_s h_t(s, x') p(x') \pi(x', s) dx' \\
&= \E_X[\E_S[h_1(x,S)|X=x, T=1]] - \E_X[\E_S[h_0(x,S)|X=x, T=0]] = \tau_Y .
\end{align*}
However, if we are able to accurately estimate $h_1(x,s)$ and $h_0(x,s)$ ahead of the trial, there is little need for the trial in the first place, since their \emph{averages} are precisely the goal of most experiments. 

\newpage
\section{Experimental details}
\label{app:experiment_details}
\paragraph{Simulation Studies}
We generate data starting using a binary treatment $T$, a vector-valued covariate $X$, and outcome $Y$.  We generate three types of surrogate variables, $S_{\med}$, $S_{\leaf}$, $S_{\proxy}$, to study how different baseline models learn about surrogates. $S_{\med}$ variables are mediators of the effect between $T$ and $Y$, $S_{\leaf}$ variables are influenced by $T$ but do not have an effect on $Y$, while the $S_{\proxy}$ variables influence $Y$ but are  not influenced by the treatment $T$. The structural equations are given as

\begin{align} \label{eq:scm_experiment}
\begin{split}
X &\sim \mc{N}(\mu_x,\Sigma_X) \\
T &= \operatorname{Bernoulli}\left(\sigma(W^\top_{X\to T} X +b_t + \e_T)\right) \\
S_{\med} &= b_{s_0}+ W^\top_{x\to S_0} X +  \left(W^\top_{X\to S_0}X + b_{s_0} \right) T+ \e_{S_0} \\
S_{\leaf} &= b_{s_1}+ W^\top_{x\to S_1} X +  \left(W^\top_{X\to S_1}X + b_{s_1} \right) T+ \e_{S_1} \\
S_{\proxy} &= b_{s_2}+ W^\top_{x\to S_2} X +  \e_{S_2} \\
Y &= b_y + W_{X \to Y}^\top X + \phi(S_{\med}) + \phi(S_{\proxy}) +  \e_Y,
\end{split}
\end{align}

where $\sigma$ denotes the sigmoid function, and $\phi$ is a linear or a nonlinear map, depending on the scenario. We set $X$ to have  $2$ dimensions with standard normal distribution. The treatment parameters are $W_{X \to T} =[0.8, -0.6]$ with intercept term $b_t =-0.1$. Surrogate variables have a dimensionality of $3$, $2$, $2$ for $S_{\med}$, $S_{\leaf}$, and $S_{\proxy}$ respectively. We choose standard normal distributions for the noise variables $\e_T$ $\e_{S_0}$, $\e_{S_1}$, $\e_{S_2}$,  $\e_{Y}$ and the columns of the linear maps $W_{X \to S_0}$,  $W_{X \to S_1}$,  $W_{X \to S_2}$ are sampled from a uniform distribution on a hyper sphere of radius  $0.7$, $0.5$, and $0.6$ respectively. The intercept terms $b_{s_0}$, $b_{s_1}$, $b_{s_2}$, $b_y$ are sampled from $\mc{N}(0.6, 0.25)$. Each experiment contains $N=10 000$ samples. We validate on a simulated trial data where we remove the confounding effect of $X$ on $T$ and sample $T\sim \operatorname{Bernoulli}(0.5)$.
 
We modify the structural equations under six main scenarios given in~\Cref{fig:six-dags}---for instance by setting $W_{X\to Y} = 0$ for scenario (a). Furthermore each scenario has ten sub-scenarios, allowing us to study the baseline models under unobserved confounders and differing scales of relationships between the surrogates. In some scenarios, we add an unobserved confounder variable $U\sim \mc{N}(0, 1)$ influencing $T$, $S$, and $Y$, in other scenarios we change the relative scales of $S_{\med}$ and $S_{\proxy}$ and add non-linear square terms, $\phi(z) = z^2$, for their effects on the outcome. In total, this yields 60 experimental configurations.

\paragraph{Models}
We implement \algbound{} by fitting a weighted regression of $\hat h$ on the surrogate vector $S$, where observations are weighted by the stabilized $w^2$, with propensities clipped between $(0.3, 0.7)$. In the synthetic experiments, we use a weighted Lasso implementation to encourage sparse, stable surrogate models. For the IHDP application, we instead use a weighted ordinary least squares (OLS) specification with sample weights, motivated by the more limited sample size and the resulting instability of regularization-based model selection in that setting. Propensity scores $e(x)$ and surrogate scores $\rho(x,s)$ are estimated using $L_2$-regularized logistic regression. For \algsamp{}, we first fit nuisance random forest models for $\hat{\E}[S|x,T=t]$ and sample $L=50$ residuals of the fitted model via bootstrapping to simulate the conditional $\hat{p}(S|x, T=t)$. We fit a PyTorch~\cite{pytorch} linear regression model on the potential outcomes $\hat{s}(0), \hat{s}(1)$ sampled from the conditional distribution (see \Cref{alg:methods}). In our implementation, we use an unregularized linear specification (no $L_1$ penalty). For both methods, we estimate the surrogate index $\hat{h}(x,s)$ by fitting a gradient boost model. All models are implemented using scikit-learn implementation~\cite{scikitlearn}. Random forest is fitted using default parameters, the parameters for gradient boost model is given in~\Cref{table:hgb-params}. 

In addition for \algbound{} we implemented optimal tree-based model variant where we replace the linear lasso regression with  
an optimal tree model \cite{aglin2021pydl8}. 
Since DL8.5 operates on binary input, we first discretize each surrogate $S_j$ into a set of cumulative threshold indicators based on empirical quantiles.
\[
Z_{j\ell} = \mathbb{I}\{S_j \le q_{j\ell}\}, \qquad q_{j\ell}\in\mathcal{Q},
\]
and learn the tree on the resulting binary feature vector $Z$. Given predictions $\hh_i = \hh(X_i,S_i)$ and nonnegative weights $w_i$ (defined by the chosen bound-weighting scheme), DL8.5 minimizes the squared error within the leaf. For a leaf $L$, the optimal leaf prediction is the weighted mean
\[
\widehat\mu(L) \;=\; \frac{\sum_{i\in L} w_i\,\hh_i}{\sum_{i\in L} w_i},
\]
and the corresponding leaf loss is
\[
\mathrm{Err}(L) \;=\; \sum_{i\in L} w_i\bigl(\hh_i-\widehat\mu(L)\bigr)^2.
\]
The tree objective is the sum of $\mathrm{Err}(L)$ over leaves, subject to the depth, minimum-support, and run time constraints of DL8.5.

In IHDP, for \algbound{} methods, we use the fixed propensity specification, setting $e(X)=\hat{p}(T=1)$, the empirical treatment probability in the training split. Since IHDP is a randomized experiment, treatment is independent of $X$ by design, ensuring the constant propensity is correctly specified in this setting.
\paragraph{Baseline models} We include outcome regressions, estimating $\E[Y\mid S=s]$, as well a single-surrogate regression (Reg-Sel-Reg), learned by fitting a regression, selecting the variable with the largest absolute coefficient, and regressing only on that. We also include the surrogate index $h(x,s) =\hat{\E}[Y|x, s]$ as a comparison that depends directly on pre-treatment variables, $x$. All baselines come in linear and tree model variants, with an additional gradient boosting regression with parameters in~\Cref{table:hgb-params} for the surrogate index. 
\paragraph{Hyperparameter search for tree models} For the tree models, we use a $5-$fold cross validation and use grid search over hyperparameters consisting of tree depth, necessary numbers for splitting, and number of minimum samples per leaf, and the complexity regularization parameter. See \Cref{tab:tree_grid} for the grid search parameters.
\begin{figure}[t]
\centering
\begin{subfigure}[t]{0.48\linewidth}
    \centering
    \includegraphics[width=\linewidth]{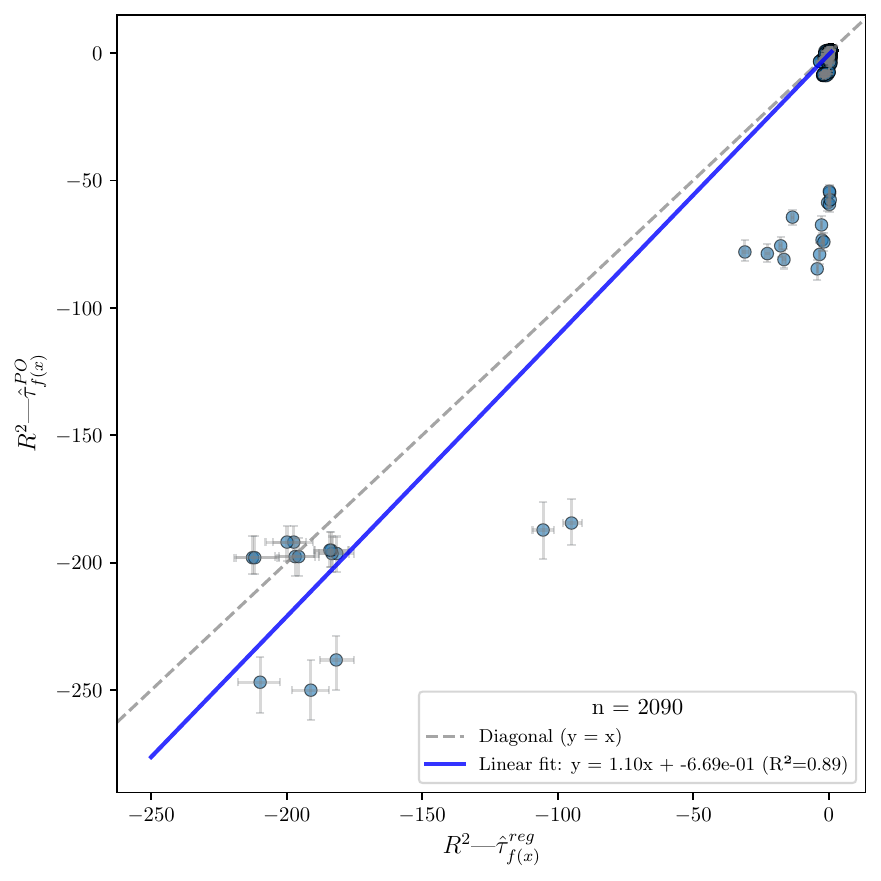}
    \caption{Composite scenarios in \Cref{tab:synth_a_e}}
    \label{fig:composite_scenarios_R2}
\end{subfigure}
\hfill
\begin{subfigure}[t]{0.48\linewidth}
    \centering
    \includegraphics[width=\linewidth]{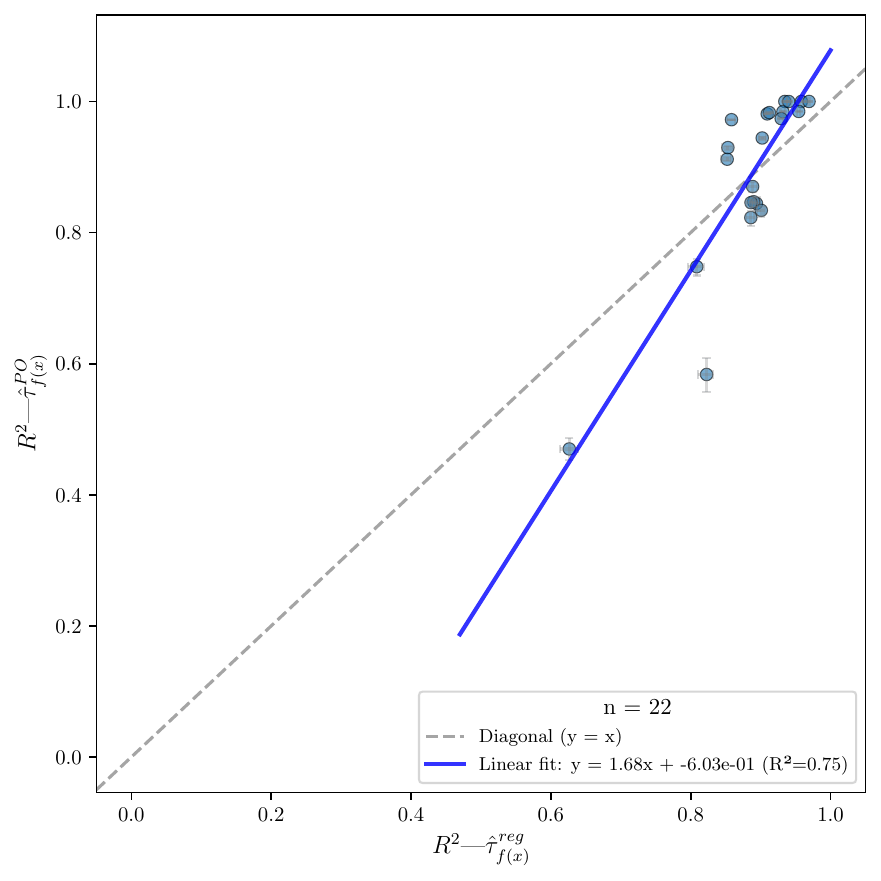}
    \caption{Linear scenarios in \Cref{apptab:synth_a_e}}
    \label{fig:linear_scenarios_R2}
\end{subfigure}
\caption{Comparison of $R^2$ computed via potential outcomes and regression across reported results. We omit results from scenarios b) and c) as the CATE is constant and $R^2$ is not meaningful.\label{appfig:r2_comp}}
\end{figure}
\begin{table}[t]
\centering
\caption{Grid search parameter space for the tree-based regression models. The grid controls tree depth, minimum leaf size, minimum split size, and cost-complexity pruning strength.}
\label{tab:tree_grid}
\begin{tabular}{ll}
\toprule
\textbf{Parameter} & \textbf{Values} \\
\midrule
\texttt{reg\_\_max\_depth} & $[3,\;4,\; \mbox{None}]$ \\
\texttt{reg\_\_min\_samples\_leaf} & $[50,\;100,\;200]$ \\
\texttt{reg\_\_min\_samples\_split} & $[100,\;200,\;400]$ \\
\texttt{reg\_\_ccp\_alpha} & $[0.0,\;10^{-3}]$ \\
\bottomrule
\end{tabular}
\end{table} 
\begin{table}[t]
\centering
\caption{Hyperparameters for the Surrogate Index Gradient Boost Model}
\begin{tabular}{ll}
\hline
\textbf{Hyperparameter} & \textbf{Default Value} \\
\hline

max\_depth             & 6 \\
min\_samples\_leaf     & 50 \\
learning\_rate         & 0.05 \\
max\_iter              & 400 \\
l2\_regularization     & 0.0 \\
early\_stopping        & True \\
validation\_fraction   & 0.1 \\
n\_iter\_no\_change    & 20 \\
\hline
\end{tabular}
\label{table:hgb-params}
\end{table}

\paragraph{Metrics} We report mean absolute error $(\operatorname{MAE}~\hat{\tau})$, for the treatment effect $\tau$ ATE, and precision in average treatment effect ($\operatorname{PEHE}= \frac1N \sum_{i=1}^N\ \left(\tau(x_i) -\hat{\tau}(x_i) \right)^2$) and coefficient of determination $R^2$ for the CATE, $\tau(x)$. We estimate $\hat{\tau}$ using the estimated average potential outcomes in the trial $\hat{\tau} = \hat{\E}^e[f(S) \mid T=1 ]-\hat{\E}^e[f(S) \mid T=0]$. However, such averaging techniques are not feasible for estimating CATE as both potential outcomes $s(1)$, $s(0)$ are not available. We estimate CATE using two methods, first by fitting linear models $g_t(x)$ for $\E_S[f(S) |x, T=t]$, and by using the ground-truth potential outcomes $s(1)$ and $s(0)$ and then estimate the CATE using
\begin{align*}
    \hat{\tau}_{f(S)}^{reg} (x) &= g_1(x) - g_0(x), \\
    \hat{\tau}_{f(S)}^{PO} (x) &= f(s(1)) - f(s(0)), \mbox{ where } s(t)\sim p(s|x,T=t),
\end{align*}
for each method, respectively. In \Cref{tab:synth_a_e} and \Cref{apptab:synth_a_e}, we report the regression based metric as the $R^2$. We give a comparison of the cases in \Cref{appfig:r2_comp}. 
\subsection{IHDP variables}
\label{app:ihdp_variables}
For lists of the variables used in IHDP, see Tables~\ref{tab:ihdp_design}--\ref{tab:ihdp_outcome}.
\begin{table}[t]
\centering
\caption{Design and treatment variables in IHDP.}
\label{tab:ihdp_design}
\small
\begin{tabular}{lll}
\toprule
Variable & Role & Description \\
\midrule
\texttt{IHDP\_Number} & ID & Infant identifier. \\
\texttt{Site\_name} & Site & Multi-site trial identifier. \\
\texttt{Treatment\_group} & $T$ & Randomized assignment (mapped to 0/1). \\
\texttt{Birth\_weight\_group} & Stratum & Birth-weight group (used for subgroup analysis). \\
\bottomrule
\end{tabular}
\end{table}
\begin{table}[h]
\centering
\caption{Post-treatment surrogate variables $S$ (SURROGATE\_SET=all).}
\label{tab:ihdp_surrogates}
\small
\begin{tabular}{p{0.42\linewidth}p{0.50\linewidth}}
\toprule
Time & Variables \\
\midrule
40 weeks & \texttt{Infant\_weight\_40w}, \texttt{Infant\_length\_40w}, \texttt{BMI\_40w} \\
4 months & \texttt{Infant\_weight\_4m}, \texttt{Infant\_length\_4m}, \texttt{BMI\_4m} \\
8 months & \texttt{Infant\_weight\_8m}, \texttt{Infant\_length\_8m}, \texttt{BMI\_8m} \\
12 months & Anthropometrics + \texttt{Bayley\_MDI\_12m}, \texttt{Bayley\_PDI\_12m}, Stein/RAND 12m scales \\
18 months & \texttt{Infant\_weight\_18m}, \texttt{Infant\_length\_18m}, \texttt{BMI\_18m} \\
24 months & Anthropometrics + \texttt{Bayley\_MDI\_24m}, \texttt{Bayley\_PDI\_24m}, Stein/RAND 24m scales \\
Year 1–2 summaries & Morbidity indices, serious conditions, hospitalizations, injuries, illnesses (\texttt{\_Y1}, \texttt{\_Y2}) \\
\bottomrule
\end{tabular}
\end{table}
\begin{table}[t]
\centering
\setlength{\tabcolsep}{13 pt}
\caption{Baseline covariates $X$ (pre-treatment; suffix \texttt{\_baseline}).}
\label{tab:ihdp_baseline}
\small
\begin{tabular}{p{0.42\linewidth}p{0.50\linewidth}}
\toprule
Variable & Description \\
\midrule
\texttt{Birth\_weight\_gm\_baseline} & Birth weight (grams). \\
\texttt{Infant\_sex\_baseline} & Sex (1=M, 2=F). \\
\texttt{Maternal\_age\_birth\_baseline} & Maternal age at birth (years). \\
\texttt{Maternal\_education\_baseline} & Education (1:9th; 2:9–12th; 3:HS; 4:some college; 5:college+). \\
\texttt{Maternal\_race}, \texttt{BLACK}, \texttt{HISPANIC} & Race/ethnicity indicators. \\
\texttt{Neonatal\_health\_index\_by\_BW\_baseline} & Neonatal health index. \\
\texttt{Analysis\_gestational\_age\_weeks\_baseline} & Gestational age (weeks). \\
\texttt{SGA\_based\_on\_ANGA\_BW\_baseline} & Small-for-gestational-age (0/1). \\
\texttt{SGA\_based\_on\_ANGA\_birth\_length\_baseline} & SGA by birth length (0/1). \\
\texttt{SGA\_based\_on\_ANGA\_birth\_head\_circ\_baseline} & SGA by head circumference (0/1). \\
\texttt{Birth\_order\_baseline} & Birth order (1,2,3+). \\
\texttt{Marital\_status\_birth\_baseline} & Marital status (1–4). \\
\texttt{Head\_circ\_birth\_cm\_baseline} & Head circumference at birth (cm). \\
\texttt{Infant\_length\_at\_birth\_baseline} & Length at birth. \\
\texttt{BMI\_at\_birth\_baseline} & Body mass index at birth. \\
\bottomrule
\end{tabular}
\end{table}
\begin{table}[t]
\centering
\caption{Primary outcome.}
\label{tab:ihdp_outcome}
\small
\begin{tabular}{lll}
\toprule
Variable & Time & Description \\
\midrule
\texttt{Stanford\_binet\_IQ\_36m} & 36 months & Stanford–Binet IQ score. \\
\bottomrule
\end{tabular}
\end{table}

\newpage
\section{Additional experimental results}
In Table~\ref{apptab:synth_a_e}, we show the results for all methods on linear versions of scenarios a--d). 
\begin{table*}[h]
\caption{Results on linear synthetic scenarios a--d), matching the graphs in Figure~\ref{fig:cdags4} and described further in Appendix~\ref{app:experiment_details}. Results are averaged over linear and nonlinear scenarios. We report $95\%$ confidence intervals computed using the standard error for each run. $^\dagger$The surrogate index does not produce a plug-in surrogate as it depends on pre-treatment variables.}
    \label{apptab:synth_a_e}
    \centering
\begin{tabular}{lcccccc}
    \toprule
     Linear Scenarios & \multicolumn{2}{c}{\footnotesize \makecell{\textbf{Case a)}\\ $\overline{ATE}$=1.83}} & \multicolumn{1}{c}{\footnotesize \makecell{\textbf{Case b)}\\ $\overline{ATE}$=2.11}} & \multicolumn{1}{c}{\footnotesize \makecell{\textbf{Case c)}\\ $\overline{ATE}$=0.00}} & \multicolumn{2}{c}{\footnotesize \makecell{\textbf{Case d)}\\ $\overline{ATE}$=18.54}} \\
    \textbf{Method} & {\footnotesize MAE $\hat{\tau}\downarrow$} & {\footnotesize $R^2 \hat{\tau}(x)\uparrow$} & {\footnotesize MAE $\hat{\tau}\downarrow$} & {\footnotesize MAE $\hat{\tau}\downarrow$} & {\footnotesize MAE $\hat{\tau}\downarrow$} & {\footnotesize $R^2 \hat{\tau}(x)\uparrow$} \\
    \midrule
    Reg.-sel.-reg. (linear) & $0.05\pm$ {\scriptsize 0.04} & $0.86 \pm$ {\scriptsize 0.01} & $0.73\pm$ {\scriptsize 0.02} & $2.45\pm$ {\scriptsize 0.03} & $1.98\pm$ {\scriptsize 0.06} & $0.85 \pm$ {\scriptsize 0.01} \\
    Reg.-sel.-reg. (tree) & $0.06\pm$ {\scriptsize 0.05} & $0.85 \pm$ {\scriptsize 0.01} & $0.74\pm$ {\scriptsize 0.02} & $2.32\pm$ {\scriptsize 0.03} & $1.40\pm$ {\scriptsize 0.08} & $0.81 \pm$ {\scriptsize 0.01} \\
    Outcome reg. (linear) & $0.16\pm$ {\scriptsize 0.04} & $0.91 \pm$ {\scriptsize 0.01} & $0.03\pm$ {\scriptsize 0.03} & $2.03\pm$ {\scriptsize 0.03} & $1.63\pm$ {\scriptsize 0.05} & $0.90 \pm$ {\scriptsize 0.00} \\
    Outcome reg. (tree) & $0.14\pm$ {\scriptsize 0.05} & $0.89 \pm$ {\scriptsize 0.01} & $0.05\pm$ {\scriptsize 0.03} & $1.03\pm$ {\scriptsize 0.05} & $0.93\pm$ {\scriptsize 0.07} & $0.89 \pm$ {\scriptsize 0.01} \\
    Surrogate index$^\dagger$ (linear) & $0.02\pm$ {\scriptsize 0.03} & $0.93 \pm$ {\scriptsize 0.00} & $0.02\pm$ {\scriptsize 0.03} & $0.01\pm$ {\scriptsize 0.01} & $0.04\pm$ {\scriptsize 0.04} & $0.96 \pm$ {\scriptsize 0.00} \\
    Surrogate index$^\dagger$ (tree) & $0.08\pm$ {\scriptsize 0.04} & $0.89 \pm$ {\scriptsize 0.01} & $0.22\pm$ {\scriptsize 0.03} & $0.72\pm$ {\scriptsize 0.04} & $0.16\pm$ {\scriptsize 0.08} & $0.90 \pm$ {\scriptsize 0.01} \\
    Surrogate index$^\dagger$ (hgb) & $0.03\pm$ {\scriptsize 0.03} & $0.93 \pm$ {\scriptsize 0.00} & $0.02\pm$ {\scriptsize 0.02} & $0.14\pm$ {\scriptsize 0.02} & $0.07\pm$ {\scriptsize 0.05} & $0.95 \pm$ {\scriptsize 0.00} \\
    \midrule
    Bound reg. (linear) & $0.14\pm$ {\scriptsize 0.05} & $0.91 \pm$ {\scriptsize 0.00} & $0.10\pm$ {\scriptsize 0.03} & $1.48\pm$ {\scriptsize 0.03} & $1.11\pm$ {\scriptsize 0.05} & $0.93 \pm$ {\scriptsize 0.00} \\
    Bound reg. (tree) & $0.13\pm$ {\scriptsize 0.05} & $0.89 \pm$ {\scriptsize 0.01} & $0.01\pm$ {\scriptsize 0.02} & $0.43\pm$ {\scriptsize 0.04} & $1.27\pm$ {\scriptsize 0.06} & $0.89 \pm$ {\scriptsize 0.01} \\
    Surrogate Sampl. (linear) & $0.03\pm$ {\scriptsize 0.03} & $0.94 \pm$ {\scriptsize 0.00} & $0.01\pm$ {\scriptsize 0.02} & $0.09\pm$ {\scriptsize 0.01} & $0.03\pm$ {\scriptsize 0.03} & $0.97 \pm$ {\scriptsize 0.00} \\
    \bottomrule
\end{tabular}
\end{table*}
\newpage
\section{Limitations of existing and alternative methods}
\label{app:baselines}
\subsection{Outcome regression is confounded in general}
\label{app:confounded_regression}

Regression of $Y$ onto $S$ to select a surrogate yields confounded estimates of the treatment effect in general. In~\Cref{prop:linear-bias} we show how such confounding bias emerges, and give a simple example scenario to highlight the bias.

\begin{thmprop}[Biased Linear Estimator]
Suppose that the outcome is given by
$Y(t) = \gamma^\top X + \delta^\top S(t) + \epsilon_Y$, with noise variable $\epsilon_Y$. A linear estimator for $\E[Y|S]$, will be a biased estimator of the treatment effect $\tau_Y$.
\label{prop:linear-bias}
\end{thmprop}

\begin{proof}
A linear model recovers $\E[Y|S] = \gamma^\top\E[X|S] + \delta^\top S$, where $E[X|S]$ is not necessarily linear. Resulting in a bias in the estimated effect
\begin{align*}
\E[Y|S(1)] - \E[Y|S(0)] &= 
\gamma^\top(\E[X|S(1)]-\E[X|S(0)]) + \delta^\top \tau_S \\
&\neq \delta^\top \tau_S = Y(1)-Y(0).
\end{align*}
\end{proof}

\paragraph{Example Scenario} Consider the following structural equations:
\begin{align*}
    X &\sim \mathcal{N}(0,1) \\
    T &\sim \mbox{Bernoulli}(0.5)\\
    S &\sim \mathcal{N}(XT + T, 1)\\
    Y &\sim \mathcal{N}(5X + S, 1)
\end{align*}

In this case, it is easy to prove that $\tau_Y = \E[Y(1)-Y(0)] = 1$. However, using $f(S) = \E[Y \mid S] = 5\E[X \mid S] + S$ as a plug-in surrogate, 
$$
\tau_{f(S)} = \E_X[\E_S[f(S) \mid X, T=1] - \E_S[f(S) \mid X, T=0]] \approx 2.43~.
$$
Note that this is not because of confounding of the effect of $T$ on $Y$ since $T$ is marginally independent of $X$ (and is causally unaffected by all other variables). This is merely because the association between $S$ and $Y$ is confounded by $X$.

\subsection{Matching ATEs $\E[\Delta_{f(S)}]$ and $\E[\Delta_Y]$ is not enough}
\label{app:ATE_matching}
We give a simple result showing why matching average treatment effects is a weak criterion for surrogate learning.
\begin{thmthm}
Assume that $X$ is an adjustment set for the effect of $T$ on $S$, that is $S(t) \indep T \mid X$ and that $Y \indep T \mid X, S$. Further, assume that, for some surrogate variable $S_i \in S$, there exists $\delta>0$, such that $|\E[w_1(X)S_i\mid T=1] - \E[w_0(X)S_i\mid T=0]| \geq \delta$, where $w_t(x) = \frac{p(T=t)}{p(T=t\mid X=x)}$ for $t\in \{0,1\}$. Then, 
$$
f(S) = \frac{\E[\Delta_Y]}{\delta} S_i  \;\;\mbox{satisfies}\;\; \E[\Delta_{f(S)}] = \E[\Delta_{Y}]~.
$$
\end{thmthm}

\begin{proof}
We show later that 
\begin{align*}
\E[Y(t)] &= \E_{X,S}\left[w_t(x)h(X, S) \mid T=1\right] \\
\E[f(S(t))] & = \E_{X,S}\left[w_t(x)f(S)\mid T=1\right].
\end{align*}
To match the ATEs of the surrogate to the primary outcome, we therefore need
$$
\E\left[w_1(X)f(S) \mid T=1\right]- \E\left[w_0(X)f(S) \mid T=0\right] = \E[\Delta_Y]
$$
As long as $\delta = \E[w_1(X)S_i\mid T=1] - \E[w_0(X)S_i\mid T=0]$ satisfies $|\delta|>0$ for some surrogate variable $S_i$, this is trivially achievable by a linear $f(S) = \alpha S_i$ since then
\begin{align*}
& \E\left[w_1(X)f(S) \mid T=1\right] - \E\left[w_0(X)f(S) \mid T=0\right] = \alpha\delta
\end{align*}
and we can define $\alpha = \E[\Delta_Y]/\delta$. In other words, as long as the treatment $T$ has an effect on \emph{any} part of $S$, we can match the average causal effect with some model $f(S)$. 
\end{proof}
Most choices of $\alpha, S_i$, however, would not be good choices for a surrogate. For example, the statistical implications of choosing a surrogate barely affected by the treatment may be large---it would require a large experiment cohort to see an effect. To select $\alpha$ optimally requires knowing $\E[\Delta_Y]$ precisely. 

\section{Additional related work}
\label{app:related}
As several definitions of surrogates rely on untestable (causal) assumptions, researchers have studied the predictiveness of various surrogates empirically~\citep{gilbert2006evaluating,elliott2015surrogacy,wang2022surrogate}. Recently, a study examining this question using surrogate outcomes in A/B tests at a large streaming service~\citep{zhang2023evaluating} found that the so-called ``surrogate index''~\citep{athey2025surrogate}, based on the definition of \citet{prentice1989surrogate}, showed strong predictive performance. In a recent work \citet{molenberghs2024statistical} suggest a causally grounded method of adopting and evaluation of surrogates in clinical trials. 
\end{document}